\documentclass[11pt]{article}
\usepackage{rotating,amsmath,amssymb,amsfonts,amsthm}
\usepackage{epsfig,mathrsfs,natbib,color,graphics}
\usepackage{graphicx,algorithm,algorithmic}
\usepackage{array,wrapfig}
\usepackage{hyperref}
\usepackage{booktabs}
\usepackage{subcaption}
\usepackage{bbm}
\def\boxit#1{\vbox{\hrule\hbox{\vrule\kern6pt \vbox{\kern6pt#1\kern5pt}
\kern6pt\vrule}\hrule}}

\newcommand{\by}{{\boldsymbol y}}
\newcommand{\bx}{{\boldsymbol x}}

\newcommand{\bb}{{\boldsymbol b}}
\newcommand{\bw}{{\boldsymbol w}}

\newcommand{\bu}{{\boldsymbol u}}

\newcommand{\mR}{\mathbb{R}}

\newcommand{\mP}{\mathbb{P}}

\newcommand{\mG}{{\cal G}}

\newcommand{\mX}{{\cal X}}

\newcommand{\bbeta}{{\boldsymbol \beta}}
\newcommand{\bgamma}{{\boldsymbol \gamma}}
\newcommand{\btheta}{{\boldsymbol \theta}}

\newtheorem{theorem}{Theorem}[section]
\newtheorem{lemma}{Lemma}[section]

\newtheorem{remark}{Remark}[section]
\begin{document}

\title{Consistent Sparse Deep Learning: Theory and Computation} 

\author{Yan Sun$^{\dag}$, Qifan Song$^{\dag}$, and Faming Liang \thanks{To whom correspondence should be addressed: Faming Liang.  F. Liang is Professor (email: fmliang@purdue.edu), Y. Sun is Graduate Student (email: sun748@purdue.edu), and Q. Song is Assistant Professor (email: qfsong@purdue.edu), Department of Statistics, Purdue University, West Lafayette, IN 47907. $^{\dag}$Sun and Song are co-first authors and they contribute equally to this paper. } }

\maketitle

\begin{abstract}
Deep learning has been the engine powering many successes of data science. However, the deep neural network (DNN), as the basic model of deep learning, is often excessively over-parameterized, causing many difficulties in training, prediction and interpretation. We propose a frequentist-like method for learning sparse DNNs and justify its consistency under the Bayesian framework: the proposed method could learn a sparse DNN with at most $O(n/\log(n))$ connections and nice theoretical guarantees such as posterior consistency, variable selection consistency and asymptotically optimal generalization bounds. In particular, we establish posterior consistency for the sparse DNN with a mixture Gaussian prior, show that the structure of the sparse DNN can be consistently determined using a Laplace approximation-based marginal posterior inclusion probability approach, and use Bayesian evidence to elicit sparse DNNs learned by an optimization method such as stochastic gradient descent in multiple runs with different initializations. The proposed method  is computationally more efficient than standard Bayesian methods for large-scale sparse DNNs.  The numerical results indicate that the proposed method can perform very well for large-scale network compression and high-dimensional nonlinear variable selection, both advancing interpretable machine learning.

\end{abstract}


\section{Introduction} 
                
During the past decade, the deep neural network (DNN) has achieved great successes in solving
many complex machine learning tasks such as pattern recognition and natural 
language processing. A key factor to the successes is its superior approximation power over the shallow one \citep{montufar2014number, telgarsky2017neural, yarotsky2017error, mhaskar2017and}. 
The DNNs used in practice may 
consist of hundreds of layers and millions of parameters, 
see e.g. \cite{DNNRes2016} on image classification. 
Training and operation of DNNs of this scale entail formidable 
computational challenges. Moreover, the DNN models with massive parameters are more easily 
overfitted when the training samples are insufficient. 
DNNs are known to have many redundant parameters \citep{glorot2011deep,YoonHwang2017, scardapane2017group, 
Denil2013, ChengChang2015, MocanuLiotta2018}. For example, \cite{Denil2013} showed 
that in some networks, only 5\% of the parameters are enough to achieve acceptable models; and 
\cite{glorot2011deep} showed that sparsity (via employing a ReLU activation function) can generally improve the training and prediction performance of the DNN. 
Over-parameterization often makes the DNN model less interpretable and miscalibrated \citep{CalibrationDNN2017}, which can 
cause serious issues in human-machine trust and thus hinder applications of artificial intelligence (AI) in human life.

The desire to reduce the complexity of DNNs naturally leads to two questions: (i) 
Is a sparsely connected DNN, also known as sparse DNN, able to approximate the 
target mapping with a desired accuracy? and (ii) how to train and determine the structure of a sparse DNN?  
This paper answers these two questions in a coherent way. The proposed method is essentially a regularization method, but justified 
under the Bayesian framework. 

 The approximation power of sparse DNNs has been studied 
 in the literature from both frequentist and Bayesian perspectives. 
 From the frequentist perspective, \cite{Bolcskei2019} quantifies 
 the minimum network connectivity that guarantees uniform approximation rates 
 for a class of affine functions; and
 \cite{Schmidt-Hieber2017Nonparametric} and \cite{Bauler2019On} characterize the approximation error of a sparsely connected neural network for H{\"o}lder smooth functions.
 From the Bayesian perspective, \cite{liang2018bayesian} established 
 posterior consistency for Bayesian shallow neural networks  under mild conditions;
 and \cite{PolsonR2018} established posterior consistency for Bayesian DNNs but 
 under some restrictive conditions such as 
  a spike-and-slab prior is used for connection weights, the activation function is ReLU, and the number of input variables keeps at an order of $O(1)$ while the sample size grows to infinity.   

 The existing methods for learning sparse DNNs are usually developed separately from 
 the approximation theory. For example, \cite{alvarez2016learning}, 
 \cite{scardapane2017group} and \cite{MaZhang2019} 
 developed some regularization methods for learning sparse DNNs;
 \cite{wager2013dropout} showed that 
 dropout training is approximately equivalent to 
 an $L_2$-regularization; 
 \cite{han2015deep} introduced a deep compression 
 pipeline, where 
pruning, trained quantization and Huffman coding  work together to reduce the storage 
requirement of DNNs; \cite{liu2015sparse} proposed a sparse decomposition method to sparsify 
convolutional neural networks (CNNs); 
\cite{frankle2018lottery} considered a lottery ticket hypothesis for selecting a sparse subnetwork; and 
 \cite{ghosh2017model} proposed to learn Bayesian sparse  neural networks via node selection with a horseshoe prior under the framework of variational inference. 
 For these methods, it is generally unclear if the resulting sparse DNN is able to provide a desired approximation accuracy 
 to the true mapping and how close in structure the sparse DNN is to the underlying true DNN. 
 
 On the other hand, there are some work which developed the approximation theory for sparse DNNs but not the associated learning algorithms, see e.g., \cite{Bolcskei2019} and  \cite{PolsonR2018}. 
 An exception is \cite{liang2018bayesian}, where the population stochastic approximation Monte Carlo  (pop-SAMC) algorithm \citep{SongWL2014} was employed to learn sparse neural networks. 
 However, since the pop-SAMC algorithm belongs to the class of traditional MCMC algorithms, where the full data likelihood needs to be evaluated at each iteration,
 it is not scalable for big data problems. Moreover, it needs to run for a large number of iterations for ensuring convergence.  As 
 an alternative to overcome the convergence 
 issue of MCMC simulations, 
 the variational Bayesian method 
 \citep{JordanVB1999} has been widely used in the machine learning community.  
 Recently, it has been 
  applied to learn Bayesian neural networks (BNNs), see e.g., 
  \cite{Mnih2014} and  \cite{Blundell2015}.  However, theoretical properties of the variational posterior of the BNN are still not well understood due to its approximation nature.
 
  This paper provides a frequentist-like method for learning sparse DNNs, 
  which, with theoretical guarantee, converges to the underlying true DNN model in probability. 
  The proposed method is to first train a dense DNN using an optimization method such as stochastic gradient descent (SGD) by maximizing 
  its posterior distribution with a mixture Gaussian prior, and then sparsify 
   its structure according to the Laplace approximation of the marginal posterior inclusion 
    probabilities. Finally, Bayesian evidence is used as 
   the criterion for eliciting sparse DNNs learned by the optimization 
   method in multiple runs with different initializations. 
  To justify consistency of the sparsified  
  DNN, we first establish posterior consistency for Bayesian DNNs 
  with mixture Gaussian priors and consistency of structure selection for Bayesian DNNs 
  based on the marginal posterior inclusion probabilities, 
  and then establish consistency of 
  the sparsified DNN via Laplace approximation to the marginal posterior inclusion probabilities. In addition, we show that the Bayesian sparse DNN has asymptotically an optimal generalization bound.

 The proposed method works with various activation functions such as 
 sigmoid, tanh and ReLU, and our theory allows the number of input variables to increase 
 with the training sample size in an exponential rate. Under regularity conditions, the proposed method learns a sparse DNN of size $O(n/\log(n))$ with nice theoretical guarantees such as posterior consistency, variable selection consistency, and asymptotically optimal generalization bound. Since, for 
 the proposed method, the DNN only needs to be trained using an optimization method,
 it is computationally much more efficient than standard Bayesian methods.  
 As a by-product, this work also provides an effective method for high-dimensional nonlinear variable selection. 
  Our numerical results indicate that the proposed method can work very well for large-scale DNN compression and  high-dimensional nonlinear variable selection. For some benchmark DNN compression examples, the proposed method produced the state-of-the-art prediction accuracy using about the same amounts of parameters as the existing methods.
In summary, this paper provides a complete treatment for sparse DNNs in both theory and computation. 

 The remaining part of the paper is organized as follows. Section 2 studies the 
 consistency theory of Bayesian sparse DNNs. 
 Section 3 proposes a computational method for training sparse DNNs. 
 Section 4 presents some numerical examples. 
 Section 5 concludes the paper with a brief discussion.

\section{Consistent Sparse DNNs: Theory}

\subsection{Bayesian Sparse DNNs with mixture Gaussian Prior} \label{priorsect}

Let $D_{n}=(\boldsymbol{x}^{(i)},y^{(i)})_{i=1,...,n}$ denote a training dataset of $n$
$i.i.d$ observations, where $\boldsymbol{x}^{(i)}\in R^{p_n}$, 
$y^{(i)}\in R$, and $p_n$ denotes the dimension of input variables and is 
assumed to grow with the training sample size $n$.  
We first study the posterior approximation theory of Bayesian sparse DNNs under the 
framework 
of generalized linear models, for which the distribution of $y$ given $\bx$ is  given by 
\[
f(y|\mu^*(\bx))=\exp\{A(\mu^*(\bx))y+B(\mu^*(\bx))+C(y)\},
\]
where $\mu^*(\bx)$ denotes a nonlinear function of $\bx$, and $A(\cdot)$,
 $B(\cdot)$ and $C(\cdot)$ are appropriately defined functions.  The theoretical results presented in this work mainly focus on logistic regression models and normal linear regression models.
For logistic regression, we have  $A(\mu^*)=\mu^*$, $B(\mu^*)=-\log(1+e^{\mu^*})$,
and $C(y)=1$.
For normal regression, by introducing an extra dispersion parameter $\sigma^2$, 
we have $A(\mu^*)=\mu^*/\sigma^2$, $B(\mu^*)=-{\mu^*}^2/2\sigma^2$ 
and $C(y)=-y^2/2\sigma^2-\log(2\pi \sigma^2)/2$.
For simplicity, $\sigma^2=1$ is assumed to be known in this paper. How to extend our results to the case that $\sigma^2$ is unknown will be discussed in Remark \ref{rem3}.

We approximate $\mu^*(\bx)$ using a DNN. Consider a DNN with $H_{n}-1$ hidden layers and $L_{h}$
hidden units at layer $h$, where $L_{H_n}=1$ for the output
layer and $L_{0}=p_{n}$ for the input layer.  
Let  $\boldsymbol{w}^{h}\in \mathbb{R}^{L_{h}\times L_{h-1}}$ and $\boldsymbol{b}^{h}\in \mathbb{R}^{L_{h}\times1}$,
$h\in\{1,2,...,H_n\}$ denote the weights and bias of  layer $h$, and let 
 $\psi^{h}:R^{L_{h}\times1}\to \mathbb{R}^{L_{h}\times1}$ denote a coordinate-wise
and piecewise differentiable activation function of layer $h$.
The DNN forms a nonlinear mapping
\begin{equation} \label{appeq}
\mu(\bbeta,\boldsymbol{x})=\boldsymbol{w}^{H_{n}}\psi^{H_{n}-1}\left[\cdots\psi^{1}\left[\boldsymbol{w}^{1}\boldsymbol{x}+\boldsymbol{b}^{1}\right]\cdots\right]+\boldsymbol{b}^{H_{n}},
\end{equation}
where $\bbeta=(\bw,\bb)=\left\{ {w}_{ij}^{h},{b}_{k}^{h}: h\in\{1,2,...,H_{n}\}, i,k\in\{ 1,...,L_{h}\}, j\in\{ 1,...,L_{h-1}\} \right\}$
denotes the collection of all weights and biases,
consisting of $K_{n}=\sum_{h=1}^{H_{n}}\left(L_{h-1}\times L_{h}+L_{h}\right)$
elements in total.
To facilitate representation of the sparse DNN, we introduce an indicator variable 
for each weight and bias of the DNN, which indicates the existence of the connection in the network.   
Let $\bgamma^{\boldsymbol{w}^{h}}$ and $\bgamma^{\boldsymbol{b}^{h}}$ 
denote the matrix and vector of the indicator variables associated with 
$\boldsymbol{w}^{h}$ and  $\boldsymbol{b}^{h}$, respectively. 
Further, we let $\boldsymbol{\gamma}=\{ \bgamma_{ij}^{\boldsymbol{w}^{h}},\bgamma_{k}^{\boldsymbol{b}^{h}}:
 h\in\{1,2,...,H_{n}\}$, $i,k\in\left\{ 1,...,L_{h}\right\} ,j\in\left\{ 1,...,L_{h-1}\right\}\} $
and $\boldsymbol{\beta}_{\boldsymbol{\gamma}}=\{ {w}_{ij}^{h},{b}_{k}^{h}:\bgamma_{ij}^{\boldsymbol{w}^{h}}=1,\bgamma_{k}^{\boldsymbol{b}^{h}}=1$ ,$h\in\{1,2,...,H_{n}\},i,k\in\left\{ 1,...,L_{h}\right\}$,
$j\in\left\{ 1,...,L_{h-1}\right\}\}$,
which specify, respectively, the structure and associated parameters for a sparse DNN.

To conduct Bayesian analysis for the sparse DNN, we consider a mixture Gaussian prior specified as follows:
\begin{equation} \label{mixprioreq1}
\bgamma_{ij}^{\boldsymbol{w}^{h}}\sim Bernoulli(\lambda_{n}), \quad \bgamma_{k}^{\boldsymbol{b}^{h}}\sim Bernoulli(\lambda_{n}),
\end{equation}
\begin{equation} \label{mixprioreq2}
{\bw}_{ij}^{h}|\bgamma_{ij}^{\boldsymbol{w}^{h}}\sim\bgamma_{ij}^{\boldsymbol{w}^{h}}N(0,\sigma_{1,n}^{2})+(1-\bgamma_{ij}^{\boldsymbol{w}^{h}})N(0,\sigma_{0,n}^{2}),\ \ 
{\bb}_{k}^{h}|\bgamma_{k}^{\boldsymbol{b}^{h}}\sim\bgamma_{k}^{\boldsymbol{b}^{h}}N(0,\sigma_{1,n}^{2})+(1-\bgamma_{k}^{\boldsymbol{b}^{h}})N(0,\sigma_{0,n}^{2}),
\end{equation}
where $h\in\{1,2,...,H_{N}\}, i\in\left\{ 1,...,L_{h-1}\right\}, j,k\in\left\{ 1,...,L_{h}\right\}$, and 
$\sigma_{0,n}^2<\sigma_{1,n}^2$ are prespecified constants. 
Marginally, we have
\begin{equation} \label{marprior}
{w}_{ij}^{h}  \sim\lambda_{n}N(0,\sigma_{1,n}^{2})+(1-\lambda_{n})N(0,\sigma_{0,n}^{2}), \ \ 
{b}_{k}^{h}  \sim\lambda_{n}N(0,\sigma_{1,n}^{2})+(1-\lambda_{n})N(0,\sigma_{0,n}^{2}). \\
\end{equation}
Typically, we set $\sigma_{0,n}^{2}$ to be a very
small value while $\sigma_{1,n}^{2}$ to be relatively large. 
When $\sigma_{0,n}^{2} \rightarrow 0$, the prior is reduced
to the spike-and-slab prior \citep{ishwaran2005spike}. Therefore, this prior can be
viewed as a continuous relaxation of the spike-and-slab prior. Such a prior has been used by 
many authors in Bayesian variable selection, see e.g., \cite{george1993variable} and \cite{SongLiang2017}. 

\subsection{Posterior Consistency} 

Posterior consistency plays a major role in validating Bayesian methods especially for 
high-dimensional models, see e.g. \cite{Jiang2007} and \cite{LiangSY2013}.  
For DNNs, since the total number of parameters $K_n$ is often much 
larger than the sample size $n$, posterior consistency 
provides a general guideline in prior setting or choosing prior hyperparameters for a class of prior 
distributions. 
Otherwise, the prior information may dominate 
data information, rendering a biased inference for the underlying true model. In what follows, we prove the posterior consistency of the DNN model with the mixture Gaussian prior (\ref{marprior}).

With slight abuse of notation, we rewrite $\mu(\bbeta,\bx)$ in (\ref{appeq})  as $\mu(\bbeta,\bgamma,\bx)$ for a sparse network by including its network structure information.
We assume $\mu^*(\bx)$ can be well approximated 
by a {\it sparse DNN} with relevant variables, and call this sparse DNN 
as the {\it true DNN} in this paper. More precisely, we define the {\it true DNN} as 
\begin{equation} \label{trueDNNeq}
(\bbeta^*,\bgamma^*)=\underset{(\bbeta,\bgamma)\in \mG_n,\, \|\mu(\bbeta,\bgamma, \bx)-\mu^*(\bx)\|_{L^2(\Omega)} \leq \varpi_n}{\operatorname{arg\,min}}|\bgamma|, 
\end{equation}
where $\mG_n:=\mG(C_0,C_1,\varepsilon,p_n,H_n,L_1,L_2,\ldots,L_{H_n})$ denotes the space of valid sparse networks satisfying condition A.2 (given below) for the given values of $H_n$, $p_n$, and $L_h$'s,  and $\varpi_n$ is some sequence converging to 0 as $n \to \infty$. 
For any given DNN $(\bbeta,\bgamma)$, 
the error $\mu(\bbeta,\bgamma,\bx)-\mu^*(\bx)$ can be generally decomposed as the network approximation error $\mu(\bbeta^*,\bgamma^*,\bx)-\mu^*(\bx)$ and the
network estimation error $\mu(\bbeta,\bgamma,\bx)-\mu(\bbeta^*,\bgamma^*,\bx)$. The $L_2$ norm of the former one is bounded by $\varpi_n$, and the order of the latter will be given in Theorem \ref{2normal}. In what follows, we will treat $\varpi_n$ as the network approximation error. 
In addition, we make the following assumptions:
 
\begin{enumerate}
 \item[A.1] The input $\bx$ is bounded by 1 entry-wisely, i.e. $\bx\in \Omega=[-1,1]^{p_n}$, and the density of $\bx$ is bounded in its support $\Omega$ uniformly with respect to $n$.
 \item[A.2] The true sparse DNN model satisfies the following conditions: 
 \begin{itemize}
 \item[A.2.1] The network structure satisfies:  
 $r_nH_n\log n$ $+r_n\log\overline L+s_n\log p_n\leq C_0n^{1-\varepsilon}$, where $0<\varepsilon<1$ is a small constant,
 $r_n=|\bgamma^*|$ denotes the connectivity of $\bgamma^*$, $\overline{L}=\max_{1\leq j\leq H_n-1}L_j$ denotes the maximum hidden layer width, 
 $s_n$ denotes the input dimension of $\bgamma^*$.
 
 \item[A.2.2] The network weights are polynomially bounded:  $\|\bbeta^*\|_\infty\leq E_n$, where 
  $E_n=n^{C_1}$ for some constant $C_1>0$.
 \end{itemize}
 \item[A.3] The activation function $\psi$ is Lipschitz continuous with a Lipschitz constant of 1.
 \end{enumerate}
 
 Assumption A.1 is a typical assumption for posterior consistency, see e.g.,  \cite{PolsonR2018} and \cite{Jiang2007}. In practice, all bounded data can be normalized to satisfy this assumption, e.g. image data are bounded and usually normalized before training.
 Assumption A.3 is satisfied by many conventional 
activation functions such as sigmoid, tanh and ReLU.

Assumption A.2 specifies the class of DNN models that we are considering in this paper. They are sparse, while still being able to approximate many types of functions arbitrarily well as the training sample size becomes large, i.e., $\lim_{n\to \infty} \varpi_n=0$. 
The approximation power of sparse DNNs has been studied in several existing work. 
For example,  for the functions that can be  represented by an affine system, \cite{Bolcskei2019} proved that if the network parameters are bounded in absolute value by some polynomial $g(r_n)$, i.e. $||\bbeta^*||_{\infty} \leq g(r_n)$, then the approximation error $\varpi_n = O(r_n^{-\alpha^*})$ for some constant $\alpha^*$. To fit this this result into our framework, we can let $r_n \asymp n^{(1-\epsilon)/2}$ for some $0<\epsilon <1$, $p_n = d$ for some constant $d$, $H_n < r_n + d$ and $\bar{L} < r_n$ (i.e. the setting given in Proposition 3.6 of \cite{Bolcskei2019}). 
Suppose that the degree of $g(\cdot)$ is $c_2$, i.e. $g(r_n) \prec r_n^{c_2}$, then $\|\bbeta^*\|_{\infty} \prec n^{c_2(1-\epsilon)/2} \prec n^{C_1} = E_n$ for some constant $C_1 > c_2(1-\epsilon)/2$. Therefore, Assumption A.2 is satisfied with the approximation error $\varpi_n = O(r_n^{-\alpha^*}) = O(n^{-\alpha^*(1-\epsilon)/2}) \stackrel{\Delta}{=} O(n^{-\varsigma})$ (by defining $\varsigma=\alpha^*(1-\epsilon)/2$), which goes to 0 as $n\to \infty$. 
In summary, the minimax rate in
$\sup_{\mu^*(\bx) \in \mathcal{C}}\, \inf_{(\bbeta, \bgamma) \in \mathcal{G}}  \|\mu(\bbeta, \bgamma, \bx) - \mu^*(\bx) \|_{L^2(\Omega)} \in 
\mathcal{O}(n^{-\varsigma})$
can be achieved by sparse DNNs under our assumptions, where 
$\mathcal{C}$ denotes the class of functions represented by an affine system.

 Other than affine functions, our setup for the sparse DNN also matches the approximation theory for many other types of functions. For example, Corollary 3.7 of \cite{petersen2018optimal} showed that for a wide class of piecewise smooth functions with a fixed input dimension, a fixed depth ReLU network can achieve an $\varpi_n$-approximation with $\log(r_n)=O(-\log\varpi_n)$ and $\log E_n =O(-\log\varpi_n)$. This result satisfies condition A.2 by setting $\varpi_n=O(n^{-\varsigma})$ for some constant $\varsigma>0$. As another example, Theorem 3 of \cite{Schmidt-Hieber2017Nonparametric} (see also lemma 5.1 of \cite{PolsonR2018}) proved  that any bounded $\alpha$-H{\"o}lder smooth function $\mu^*(\bx)$ can be approximated by a sparse ReLU DNN with the network approximation error
 $\varpi_n=O(\log(n)^{\alpha/p_n}n ^{-\alpha/(2\alpha+p_n)})$ 
 for some $H_n\asymp \log n\log p_n$, $L_j\asymp p_n n ^{p_n/(2\alpha+p_n)}/\log n$, $r_n=O(p_n^2\alpha^{2p_n} n ^{p_n/(2\alpha+p_n)}\log p_n$), and $E_n=C$ for some fixed constant $C>0$. This result also satisfies condition A.2.2 as long as $p_n^2\ll \log n$.

  It is important to note that there is a fundamental difference between the existing neural network approximation theory and ours. 
 In the existing neural network approximation theory, no data is involved and a small network 
 can potentially achieve an arbitrarily small 
 approximation error by allowing the connection weights to take values in an unbounded space.  
  In contrast, in our theory, the network approximation error, the network size, and the bound of connection weights are all linked to the training sample size. 
 A small network  approximation error is required only when the training sample size is large; 
 otherwise, over-fitting might be a concern from the point of view of statistical modeling. In the practice of modern neural networks, the depth and width have been increased without much scruple. These increases reduce the training error, improve the  generalization performance under certain regimes \citep{Doubledescent2020}, but negatively affect model calibration \citep{CalibrationDNN2017}. We expect that our theory can tame the powerful neural networks into the framework of statistical modeling; that is, by selecting an appropriate network size according to the training sample size, the proposed method can generally improve the generalization and calibration of the DNN model while controlling the training error to a reasonable level. The calibration of the sparse DNN will be explored elsewhere.

Let $P^{*}$ and
$E^{*}$ denote the respective probability measure and expectation for data $D_{n}$.  
Let $d(p_{1},p_{2})=\left(\int\left[p_{1}^{\frac{1}{2}}(\boldsymbol{x},y)-p_{2}^{\frac{1}{2}}(\boldsymbol{x},y)\right]^{2}dyd\boldsymbol{x}\right)^{\frac{1}{2}}$ 
 denote the Hellinger distance between two density functions 
 $p_{1}(\boldsymbol{x},y)$ and $p_{2}(\boldsymbol{x},y)$.
Let $\pi(A\mid D_{n})$ be the posterior probability of an event $A$. The following theorem establishes posterior consistency for sparse DNNs 
under the mixture Gaussian prior (\ref{marprior}).

\begin{theorem}\label{2normal}
 Suppose Assumptions A.1-A.3 hold. If the mixture Gaussian prior (\ref{marprior}) satisfies the conditions:
 $\lambda_n = O( 1/\{K_n[n^{H_n}(\overline Lp_n)]^{\tau}\})$ for some constant $\tau>0$,
  $E_n/\{H_n\log n+\log \overline L\}^{1/2}  \lesssim \sigma_{1,n} \lesssim n^{\alpha}$ for some 
  constant $\alpha>0$, and 
$\sigma_{0,n}  \lesssim \min\big\{ 1/\{\sqrt{n} K_n (n^{3/2} \sigma_{1,0}/H_n)^{H_n}\}$, 
 $1/\{\sqrt{n} K_n (n E_n/H_n)^{H_n}\} \big\}$, 
 then there exists an error sequence 
 $\epsilon_n^2 =O(\varpi_n^2)+O(\zeta_n^2)$ 
 such that $\lim_{n\to \infty} \epsilon_n= 0$ and  
 $\lim_{n\to \infty} n\epsilon_n^2= \infty$, and 
 the posterior distribution satisfies
 \begin{equation}\label{postcon}
\begin{split}
 & P^*\left\{ \pi[d(p_{\bbeta},p_{\mu^*}) > 4 \epsilon_n |D_n] \geq 2 e^{-cn \epsilon_n^2} \right\} 
  \leq 2 e^{-cn \epsilon_n^2},\\
 &  E_{D_n}^* \pi[d(p_{\bbeta},p_{\mu^*}) > 4 \epsilon_n | D_n] \leq 4 e^{-2cn \epsilon_n^2},
 \end{split}
\end{equation}
 for sufficiently large $n$, where $c$ denotes a constant,  
$\zeta_n^2=[r_nH_n\log n+r_n\log \overline L+s_n\log p_n]/n$, 
$p_{\mu^*}$ denotes the underlying true data distribution, and
$p_\bbeta$ denotes the data distribution reconstructed by the Bayesian DNN based on its posterior samples. 

\end{theorem}



The proof of Theorem \ref{2normal} can be found in the supplementary material. Regarding this theorem, we have a few remarks:
 
 \begin{remark} \label{rem1}  
  Theorem \ref{2normal} provides a posterior contraction rate $\epsilon_n$ for the sparse BNN. The contraction rate contains two components, $\varpi_n$ and   $\zeta_n$, where $\varpi_n$, as defined previously, represents the network approximation error, and $\zeta_n$ represents the network estimation error measured in Hellinger distance.  Since the estimation error $\zeta_n$ grows with the network connectivity $r_n$,  there is a trade-off between the network approximation error and the network estimation error. A larger network has a lower approximation error and a higher estimation error, and vice versa. 
 \end{remark}
 
  \begin{remark} \label{rem2}
 Theorem \ref{2normal} implies that given a training sample size $n$, the proposed method can learn a sparse neural network with at most $O(n/\log(n))$ connections. Compared to the fully connected DNN, the sparsity of the proposed BNN enables  some theoretical guarantees for its performance. The sparse BNN has nice theoretical properties, such as posterior consistency, variable selection consistency, and asymptotically optimal generalization bounds, which are beyond the ability of general neural networks. The latter two properties will be established in Section \ref{selectsection} and 
 Section \ref{predictsection}, respectively.  
 \end{remark}

\begin{remark} \label{rem3} 
 Although Theorem \ref{2normal} is proved by assuming $\sigma^2$ is known, it can be easily extended to the case that $\sigma^2$ is unknown by assuming an inverse gamma prior $\sigma^2\sim \mbox{IG}(a_0,b_0)$ for some constants $a_0, b_0>0$.
 If a relatively uninformative prior is desired, one can choose $a_0\in(0,1)$ such that the inverse gamma prior is very diffuse with a non-existing mean value.
 However, if $a_0=b_0=0$, i.e., the Jeffreys prior $\pi(\sigma^2) \propto 1/\sigma^2$, the posterior consistency theory established Theorem \ref{2normal} might not hold any more.
In general, to achieve posterior consistency, the prior is required, at least in the framework adopted by the paper, to satisfy two conditions \citep{ghosal2000convergence,Jiang2007}: (i) a not too little prior probability is placed over the neighborhood of the true density, and (ii) a very little prior probability is placed outside of a region that is not too complex.   
 Obviously, the Jeffreys prior and thus the joint prior of $\sigma^2$ and the regression coefficients do not satisfy neither of the two conditions. We note that the inverse gamma prior 
 $\sigma^2\sim IG(a_0,b_0)$ has long been used 
 in Bayesian inference for many different statistical models, such as linear regression \citep{george1997approaches}, nonparametric regression \citep{kohn2001nonparametric}, and Gaussian graphical models \citep{dobra2004sparse}. 
 \end{remark}

\subsection{Consistency of DNN Structure Selection} \label{selectsection}

This section establishes consistency of DNN structure selection under posterior consistency. 
It is known that the DNN model is generally nonidentifiable due to the symmetry of the network structure. For example, the approximation $\mu(\bbeta,\bgamma, \bx)$ can be invariant if one permutes the orders of certain hidden nodes, simultaneously changes the signs of certain weights and biases if $tanh$ is used as the activation function, or re-scales certain weights and bias if
Relu is used as the activation function. However, by introducing appropriate constraints, see e.g., \cite{pourzanjani2017improving} and \cite{liang2018bayesian}, we can define a set of neural networks such that any possible neural networks can be represented by one and only one neural network in the set via nodes permutation, sign changes, weight rescaling, etc.
Let $\Theta$ denote such set of DNNs, where each element in $\Theta$ can be viewed as an equivalent class of DNN models. Let $\nu(\bgamma,\bbeta) \in \Theta$ be an operator 
that maps any 
neural network to $\Theta$ via appropriate transformations such as nodes permutation, sign changes, weight rescaling, etc. 
 To serve the purpose of structure selection in the space $\Theta$, 
 we consider the marginal posterior inclusion
 probability approach proposed in \cite{LiangSY2013} for high-dimensional variable selection.
 
 For a better description of this approach, we reparameterize $\bbeta$ and $\bgamma$ as  $\bbeta=(\bbeta_1,\bbeta_2,\ldots,\bbeta_{K_n})$ and $\bgamma=(\bgamma_1,\bgamma_2,\ldots,\bgamma_{K_n})$, respectively,   
 according to their elements. Without possible confusions, we will often 
 use the indicator vector $\bgamma$ and the active set $\{i: \bgamma_i=1, i=1,2,\ldots, K_n\}$ exchangeably; that is,  
 $i \in \bgamma$ and $\bgamma_i=1$ are equivalent.
 In addition, we will treat the connection weights $\bw$ and the hidden 
 unit biases $\bb$ equally; that is, they will not be distinguished in $\bbeta$ and $\bgamma$. 
 For convenience, we will call each element of $\bbeta$ and $\bgamma$ a `connection' 
 in what follows.

 \subsubsection{Marginal Posterior Inclusion Probability Approach}

 For each connection $c_i$, we define its marginal posterior inclusion probability by 
 \begin{equation} \label{marceq}
  q_i=\int \sum_{\bgamma} e_{i|\nu(\bgamma,\bbeta)} \pi(\bgamma|\bbeta,D_n) 
   \pi(\bbeta|D_n) d\bbeta, \quad i=1,2,\ldots, K_n,
 \end{equation}
 where $e_{i|\nu(\bgamma,\bbeta)}$ is the indicator for the existence of 
 connection $c_i$ in the network $\nu(\bgamma,\bbeta)$. Similarly, we define
 $e_{i|\nu(\bgamma^*,\bbeta^*)}$ as the indicator for the existence of connection $c_i$ in the true model $\nu(\bgamma^*,\bbeta^*)$. 
 The proposed approach is to choose the connections whose marginal posterior inclusion
 probabilities are greater than a threshold value $\hat{q}$; that is, setting
 $\hat{\bgamma}_{\hat{q}}=\{i: q_i > \hat{q}, i=1,2,\ldots, K_n\}$ as
 an estimator of $\bgamma_*=\{i: e_{i|\nu(\bgamma^*,\bbeta^*)}=1, i=1,\ldots, K_n\}$, 
 where $\bgamma_*$ can be viewed as the uniquenized true model. 
To establish the consistency of $\hat{\bgamma}_{\hat{q}}$, an identifiability
 condition for the true model is needed.
 Let $A(\epsilon_n)=\{\bbeta: d(p_\bbeta,p_{\mu^*})\geq \epsilon_n\}$.
Define
\[
\rho(\epsilon_n) = \max_{1\leq i\leq K_n} \int_{A(\epsilon_n)^c}   \sum_{\bgamma} |e_{i|\nu(\bgamma,\bbeta)}-e_{i|\nu(\bgamma^*,\bbeta^*)}|\pi(\bgamma|\bbeta, D_n) 
 \pi(\bbeta|D_n)d\bbeta, 
\]
 which measures the structure difference between the true model and the sampled models 
 on the set $A({\epsilon_n})^c$. Then the identifiability condition
 can be stated as follows:
  \begin{enumerate}
     \item[B.1] $\rho(\epsilon_n) \to 0$, as $n\to \infty$ and $\epsilon_n\to 0$.
 \end{enumerate}
That is, when $n$ is sufficiently large, if a DNN has approximately the same probability distribution 
 as the true DNN, then the structure of the DNN, after mapping into the parameter space $\Theta$, must coincide with that of the true DNN.  Note that this identifiability is different from the one mentioned at the beginning of the section. The earlier one is only with respect to structure and parameter rearrangement of the DNN. 
 Theorem \ref{Selectlem} concerns consistency of $\hat{\bgamma}_{\hat{q}}$ and its sure screening property,
 whose proof is given in the supplementary material.  

 \begin{theorem} \label{Selectlem} Assume that the conditions of Theorem \ref{2normal} and the identifiability
  condition B.1 hold. Then 
 \begin{itemize}
 \item[(i)]  
$\max_{1\leq i\leq K_n}\{|q_i-e_{i|\nu(\bgamma^*,\bbeta^*)}|\}\stackrel{p}{\to} 0$, 
where $\stackrel{p}{\to}$ denotes convergence  in probability;
  
  \item[(ii)] (sure screening) 
   $P(\bgamma_* \subset \hat{\bgamma}_{\hat{q}}) \stackrel{p}{\to} 1$ for any 
    pre-specified $\hat{q} \in (0,1)$. 
 \item[(iii)] (Consistency) 
   $P(\bgamma_* =\hat{\bgamma}_{0.5}) \stackrel{p}{\to} 1$.
  \end{itemize}
 \end{theorem}

 For a network $\bgamma$, it is easy to identify the relevant variables.  
Recall that $\bgamma^{\bw^{h}} \in \mathbb{R}^{L_{h}\times L_{h-1}}$ denotes the connection indicator matrix of layer $h$. Let 
\begin{equation} \label{newreplyeq2}
\bgamma^{\bx} = \bgamma^{\bw^{H_n}}  \bgamma^{\bw^{H_n-1}}\cdots  \bgamma^{\bw^{1}}\in \mathbb{R}^{1\times p_n},
\end{equation}
and let $\bgamma^{\bx}_i$ denote the $i$-th element of $\bgamma^{\bx}$. It is easy to see that if $\bgamma^{\bx}_i>0$ then the variable $\bx_i$ is effective in the network $\bgamma$, and $\bgamma^{\bx}_i=0$ otherwise. 
 Let $e_{\bx_i|\nu(\bgamma^*,\bbeta^*)}$ be the indicator for the effectiveness of variable $\bx_i$ in the network $\nu(\bgamma^*,\bbeta^*)$,
and let $\bgamma_*^{\bx}=\{i: e_{\bx_i|\nu(\bgamma^*,\bbeta^*)}=1, i=1,\ldots, 
p_n \}$ denote the set of true variables. 
 Similar to (\ref{marceq}), we can define the marginal inclusion probability for each variable:
 \begin{equation} \label{marxeq}
  q_i^{\bx}=\int \sum_{\bgamma} e_{\bx_i|\nu(\bgamma,\bbeta)} \pi(\bgamma|\bbeta,D_n) 
   \pi(\bbeta|D_n) d\bbeta, \quad i=1,2,\ldots, p_n,
 \end{equation}
 Then we can select the variables whose marginal posterior inclusion probabilities greater than a threshold $\hat{q}^{\bx}$, e.g., setting $\hat{q}^{\bx}=0.5$.
As implied by (\ref{newreplyeq2}), the consistency of structure selection implies consistency of variable selection.

 It is worth noting that the above variable selection consistency result is with respect to the relevant  variables defined by the true network
$\bgamma^*$. 
To achieve the variable selection consistency with respect to the relevant variables of $\mu^*(\bx)$, some extra assumptions are needed in defining $(\bbeta^*, \bgamma^*)$. How to specify these assumptions is an open problem and we would leave it to readers.  However, as shown by our simulation example, 
 the sparse model $(\bbeta^*,\bgamma^*)$ 
 defined in (\ref{trueDNNeq}) works well, which correctly identifies all the 
 relevant variables of the underlying nonlinear system.


\subsubsection{Laplace Approximation of Marginal Posterior Inclusion Probabilities}

Theorem \ref{Selectlem} establishes the consistency of DNN structure selection based on the marginal posterior inclusion probabilities. To obtain Bayesian estimates of the marginal posterior inclusion probabilities, 
intensive  Markov Chain Monte Carlo (MCMC) simulations are usually required.
Instead of performing MCMC simulations, we propose to approximate the marginal posterior inclusion probabilities using the Laplace method based on the DNN model 
trained by an optimization 
method such as SGD. Traditionally, 
such approximation is required to be performed at the 
{\it maximum a posteriori} (MAP) estimate of the DNN. However, finding the MAP 
for a large DNN is not computationally guaranteed, as there can be many local minima 
on its energy landscape. To tackle this issue, we proposed a 
Bayesian evidence method, see Section \ref{Sectcomp} for the detail,
for eliciting sparse DNN models learned by an optimization method in multiple runs 
with different initializations. 
Since conventional optimization methods such as SGD can be used to 
train the DNN here, the proposed method is computationally much more efficient than 
the standard Bayesian method. 
More importantly, as explained in Section \ref{Sectcomp}, 
consistent estimates of the marginal posterior inclusion probabilities 
might be obtained at a local 
maximizer of the log-posterior instead of the MAP estimate. In what follows, we justify the validity of Laplace approximation for marginal posterior inclusion probabilities.

Based on the marginal posterior distribution $\pi(\bbeta|D_n)$, 
the marginal posterior inclusion probability $q_i$ of connection $c_i$ can be re-expressed as
\[
q_i= \int \pi(\bgamma_i=1|\bbeta) \pi(\bbeta|D_n) d \bbeta, \quad i=1,2,\ldots, K_n. 
\]
Under the mixture Gaussian prior, it is easy to derive that 
\begin{equation} \label{Lapprob}
\pi(\bgamma_i=1|\bbeta)= \tilde{b}_i/(\tilde{a}_i+\tilde{b}_i),
\end{equation}
where 
\[
\tilde{a}_i=\frac{1-\lambda_n}{\sigma_{0,n}} \exp\{-\frac{\bbeta_i^2}{2 \sigma_{0,n}^2}\}, \quad \tilde{b}_i=\frac{\lambda_n}{\sigma_{1,n}} \exp\{-\frac{\bbeta_i^2}{2 \sigma_{1,n}^2}\}.
\]
Let's define
\begin{equation} \label{log-posterioreq}
h_{n}(\bbeta)=\frac{1}{n}\sum_{i=1}^{n}\log(p(y_{i},\boldsymbol{x}_{i}|\bbeta))+\frac{1}{n}\log(\pi(\bbeta)), 
\end{equation}
where $p(y_i,\boldsymbol{x}_{i}|\bbeta)$ denotes the likelihood function of the observation $(y_i,\bx_i)$ and $\pi(\bbeta)$ denotes the prior as specified in (\ref{marprior}). Then $\pi(\bbeta|D_n) = \frac{ e^{nh_{n}(\bbeta)}}{\int e^{nh_{n}(\bbeta)}d\bbeta}$ 
and, for a function $b(\bbeta)$, the posterior expectation is given
by $\frac{\int b(\bbeta)e^{nh_{n}(\bbeta)}d\bbeta}{\int e^{nh_{n}(\bbeta)}d\bbeta}$.
Let $\hat{\bbeta}$ denote a strict local maximum of $\pi(\bbeta|D_n)$. Then 
$\hat{\bbeta}$ is also a local maximum of $h_n(\bbeta)$. 
Let $B_{\delta}(\bbeta)$ denote an Euclidean ball of radius $\delta$
centered at $\bbeta$. Let $h_{i_{1},i_{2},\dots,i_{d}}(\bbeta)$
 denote the $d$-th order partial derivative $\frac{\partial^{d}h(\bbeta)}{\partial\bbeta_{i_{1}}\partial\bbeta_{i_{2}}\cdots\partial\bbeta_{i_{d}}}$, 
 let $H_{n}(\bbeta)$ denote the Hessian matrix of $h_{n}(\bbeta)$, let
 $h_{ij}$ denote the $(i,j)$-th component of the Hessian matrix, and  
 let $h^{ij}$ denote the $(i,j)$-component of the inverse of the Hessian matrix. 
 Recall that $\bgamma^{*}$  denotes the set of indicators for the 
 connections of the true sparse DNN,  
 $r_n$ denotes the size of the true sparse DNN, 
 and $K_n$ denotes the size of the fully connected DNN. 
 The following theorem justifies the Laplace approximation of 
 the posterior mean for a bounded function $b(\bbeta)$.

\begin{theorem} \label{MAPthem1}
Assume that there exist positive numbers $\epsilon$, $M$, $\eta$, and $n_0$ such that for any $n>n_0$,  the function $h_n(\bbeta)$ in (\ref{log-posterioreq}) 
satisfies the following conditions:
\begin{enumerate}
    \item[C.1]  
$|h_{i_1,\dots,i_d}(\hat{\bbeta})|<M$ 
hold for any $\bbeta\in B_{\epsilon}(\hat{\bbeta})$ and any $1\leq i_{1},\dots,i_{d}\leq K_{n}$, 
where $3 \leq d \leq 4$. 


\item[C.2] $|h^{ij}(\hat{\bbeta})|<M$ if
$\bgamma_{i}^*=\bgamma_{j}^*=1$
 and $|h^{ij}(\hat{\bbeta})| = O(\frac{1}{K_{n}^{2}})$ otherwise.

\item[C.3] $\det(-\frac{n}{2\pi}H_{n}(\hat{\bbeta}))^{\frac{1}{2}}\int_{\mR^{K_{n}}\setminus B_{\delta}(\hat{\bbeta})}
e^{n(h_{n}(\bbeta)-h_{n}(\hat{\bbeta}))}d\bbeta=O(\frac{r_{n}^{4}}{n})=o(1)$ for any 
$0<\delta<\epsilon$.
\end{enumerate}
For any bounded function $b(\bbeta)$, if $|b_{i_1,\dots,i_d}(\bbeta)|=|\frac{\partial^{d} b(\bbeta)}{\partial\bbeta_{i_{1}}\partial\bbeta_{i_{2}}\cdots\partial\bbeta_{i_{d}}}| <M$
holds for any $1 \leq d \leq 2$ and  any $1\leq i_{1},\dots,i_{d}\leq K_{n}$, then for the posterior 
mean of $b(\bbeta)$, we have
\[
\frac{\int b(\bbeta)e^{nh_{n}(\bbeta)}d\bbeta}{\int e^{nh_{n}(\bbeta)}d\bbeta}=b(\hat{\bbeta})+O\left(\frac{r_{n}^{4}}{n}\right).
\]
\end{theorem}

 Conditions C.1 and C.3 are typical conditions for Laplace approximation, see e.g., \cite{geisser1990validity}. Condition C.2 requires the inverse Hessian to have very small values for the elements corresponding to the false connections. 
To justify condition C.2, we note that for a multivariate normal distribution, the inverse Hessian is its covariance matrix. 
Thus, we expect that for the weights with small variance, their corresponding elements in the inverse Hessian matrix would be small as well. The following lemma quantifies the variance of the weights for the false connections.

\begin{lemma}\label{priorjust}
Assume that $\sup_n \int {|\bbeta_i|}^{2+\delta}\pi(\beta_i|D_n) d\bbeta_i \leq C < \infty \  a.s.$ for some constants $\delta > 0$ and $C > 0$ and $\rho(\epsilon_n) \asymp \pi(d(p_{\bbeta},p_{\mu^{*}})\geq \epsilon_n|D_n)$, where $\rho(\epsilon_n)$ is defined in
 Condition B.1. Then with an appropriate choice of prior hyperparameters and $\epsilon_n$, 
 $P^{*}\{E(\bbeta_i^2|D_n) \prec  \frac{1}{K_n^{2H_n-1}}\}  \geq 1-2e^{-n\epsilon_n^2/4}$ holds
 for any false connection $c_i$ in $\bgamma^*$ (i.e., $\bgamma_{i}^{*}=0$).
\end{lemma}

In addition, with an appropriate choice of prior hyperparameters, we can also show that $\pi(\bgamma_i=1|\bbeta)$ satisfies all 
the requirements of $b(\bbeta)$ in Theorem \ref{MAPthem1} with a probability tending to 1 as 
$n\to \infty$ (refer to Section 2.4 of the supplementary material for the detail). 
Then, by Theorem \ref{MAPthem1}, $q_k$ and $\pi(\bgamma_i=1|\hat{\bbeta})$ are approximately the same as $n\to \infty$, where $\pi(\bgamma_i=1|\hat{\bbeta})$ is as defined in (\ref{Lapprob}) but with $\bbeta$ replaced by $\hat{\bbeta}$. Combining with Theorem \ref{Selectlem}, we have that  $\pi(\bgamma_i=1|\hat{\bbeta})$ is a consistent estimator 
of $e_{i|\nu(\bgamma^*,\bbeta^*)}$.

\subsection{Asymptotically Optimal Generalization Bound} \label{predictsection}

This section shows the sparse BNN has asymptotically an optimal generalization bound. 
First, we introduce a PAC Bayesian bound due to
\cite{McAllester99b,McAllester99a}, where the
 acronym PAC stands for Probably Approximately Correct.
It states that with an arbitrarily high probability, the performance (as provided by a loss function) of a learning
algorithm is upper-bounded by a term decaying to an optimal value as more data is collected (hence ``approximately correct''). 
PAC-Bayes has proven over the past two decades 
to be a powerful tool to derive theoretical guarantees for many machine learning algorithms.
 
\begin{lemma}[PAC Bayesian bound]\label{pac}
Let $P$ be any data independent distribution on the machine parameters $\bbeta$, and $Q$ be any distribution that is potentially  data-dependent and absolutely continuous with respective to $P$. If the loss function $l(\bbeta,\bx,y)\in[0,1]$, then the following inequality holds with probability $1-\delta$,
\[
\int E_{\bx,y}l(\bbeta,\bx,y)dQ\leq \int \frac{1}{n}\sum_{i=1}^nl(\bbeta,\bx^{(i)},y^{(i)})dQ+\sqrt{\frac{d_0(Q,P)+\log\frac{2\sqrt n}{\delta}}{2n}},
\]
where $d_0(Q,P)$ denotes the Kullback-Leibler divergence between $Q$ and $P$, and $(\bx^{(i)},y^{(i)})$ denotes the $i$-th observation of the dataset.
\end{lemma}

For the binary classification problem, the DNN model fits a predictive distribution as $\hat p_{1}(\bx;\bbeta):=\widehat Pr(y=1|\bx)=\mbox{logit}^{-1}(\mu(\bbeta,\bx))$ and $\hat p_{0}(\bx;\bbeta):=\widehat Pr(y=0|\bx)=1-\mbox{logit}^{-1}(\mu(\bbeta,\bx))$.
Given an observation $(\bx, y)$, we define the loss with margin $\nu>0$ as
\[
l_{\nu}(\bbeta,\bx,y) =1( \hat p_y(\bx;\bbeta) - \hat  p_{1-y}(\bx;\bbeta)< \nu ).
\]
Therefore, the empirical loss for the whole data set $\{\bx^{(i)},y^{(i)}\}_{i=1}^n$ is defined as 
$L_{emp,\nu}(\bbeta)=\sum l_\nu(\bbeta,\bx^{(i)},y^{(i)}) /n$, and the population loss is defined as 
$L_{\nu}(\bbeta)=E_{\bx,\by} l_\nu(\bbeta,\bx,y)$.

\begin{theorem}[Bayesian Generalization error for classification]\label{classbge} 
Suppose the conditions of Theorem \ref{2normal} hold.
For any $\nu>0$, when $n$ is sufficiently large, the following inequality holds with  probability greater than $1-\exp\{c_0n\epsilon_n^2\}$,
\[
\int L_0(\bbeta)d\pi(\bbeta|D_n)\leq \frac{1}{1-2\exp\{-c_1n\epsilon_n^2\}}\int L_{emp,\nu}(\bbeta)d\pi(\bbeta|D_n)+O(\epsilon_n+\sqrt{\log n/n}+\exp\{-c_1n\epsilon_n^2\}),
\]
for some $c_0$, $c_1>0$, where $\epsilon_n$ is as defined in Theorem \ref{2normal}.
\end{theorem}

Theorem \ref{classbge} characterizes the relationship between Bayesian population risk $\int L_0(\bbeta)d\pi(\bbeta|D_n)$ and Bayesian empirical risk $\int L_{emp,\nu}(\bbeta)d\pi(\bbeta|D_n)$, and implies that the difference between them is $O(\epsilon_n)$. Furthermore, this generalization performance extends to any point estimator $\hat\bbeta$, as long as $\hat\bbeta$ belongs to the dominating posterior mode.

\begin{theorem}\label{generalization1C} 
Suppose that the conditions of Theorem \ref{2normal} hold and estimation $\hat\bbeta$ belongs to the dominating posterior mode under Theorem \ref{2normal}, then
for any $\nu>0$, the following inequality holds with probability greater than $1-\exp\{c_0n\epsilon_n^2\}$,
\[
 L_0(\hat\bbeta)\leq L_{emp,\nu}(\hat\bbeta)+O(\epsilon_n),
\]
for some $c_0>0$.
\end{theorem}
It is worth to clarify that the statement ``$\hat\bbeta$ belongs to the dominating posterior mode'' means $\hat\bbeta\in B_n$ where 
$B_n$ is defined in the proof of Theorem \ref{2normal} and its posterior is greater than  $1-\exp\{-cn\epsilon_n^2\}$ for some $c>0$. Therefore, if $\hat\bbeta\sim \pi(\bbeta|D_n)$, i.e., $\hat\bbeta$ is one valid posterior sample, then with high probability, it belongs to the dominating posterior mode.
The proof of the above two theorems can be found in the supplementary material.

Now we consider the generalization error for regression models.
Assume the following additional assumptions:
\begin{itemize}
    \item[D.1] The activation function $\psi\in[-1,1]$.
    \item[D.2] The last layer weights and bias in $\bbeta^*$ are restricted to the interval $[-F_n, F_n]$ for some $F_n\leq E_n$, while $F_n \to \infty$ is still allowed as $n \to \infty$.
    \item[D.3] $\max_{\bx\in\Omega}|\mu^*(\bx)|\leq F$ for some constant $F$.
\end{itemize}
Correspondingly, the priors of the last layer weights and bias are truncated on $[-F_n, F_n]$, i.e., the two normal mixture prior (\ref{marprior}) truncated on $[-F_n, F_n]$.
By the same argument of Theorem S1 (in the supplementary material), Theorem \ref{2normal} still holds.

Note that the Hellinger distance for regression problem is defined as
\[
d^2(p_\bbeta,p_{\mu^*})=\mathbb{E}_{\bx}\left(1-\exp\left\{-\frac{[\mu(\bbeta,\bx)-\mu^*(\bx)]^2}{8\sigma^2}\right\}\right).
\]
By our assumption, for any $\bbeta$ on the prior support, $|\mu(\bbeta,\bx)-\mu^*(\bx)|^2\leq (F+\overline L F_n)^2:=\overline F^2$, thus,
\begin{equation}\label{ge4}
d^2(p_\bbeta,p_{\mu^*})\geq {C_{\overline F}}E_{\bx}|\mu(\bbeta,\bx)-\mu^*(\bx)|^2,
\end{equation}
where $C_F=[1-\exp(-4\overline F^2/8\sigma^2)]/4\overline F^2$. 
Furthermore, (\ref{postcon}) implies that with probability at least $1-2\exp\{-cn\epsilon_n^2\}$, 
\begin{equation}\label{ge5}
    \int d^2(p_{\bbeta},p_{\mu^*}) d\pi(\bbeta|D_n)\leq 16\epsilon_n^2+2 e^{-cn \epsilon_n^2}.
\end{equation}

By Combining (\ref{ge4}) and (\ref{ge5}), we obtain the following Bayesian generalization error result:

\begin{theorem} (Bayesian generalization error for regression) \label{generalization2T}
Suppose the conditions of Theorem \ref{2normal} hold. When $n$ is sufficiently large, the following inequality holds with probability at least $1-2\exp\{-cn\epsilon_n^2\}$,
\begin{equation}
\int E_{\bx}|\mu(\bbeta,\bx)-\mu^*(\bx)|^2d\pi(\bbeta|D_n) \leq [16\epsilon_n^2+2 e^{-cn \epsilon_n^2}]/C_F\asymp [\epsilon_n^2+e^{-cn \epsilon_n^2}]\overline L^2 F_n^2.
\end{equation}
\end{theorem}

Similarly, if an estimator $\hat\bbeta$ belongs to the dominating posterior mode (refer to the discussion of Theorem \ref{generalization1C} for more details), then $\hat\bbeta\in \{\bbeta: d(p_{\bbeta},p_{\mu^*})\leq 4\epsilon_n \}$ and the following result hold:
\begin{theorem} \label{generalization2C} Suppose the conditions of Theorem \ref{2normal} hold, then
\begin{equation}
 E_{\bx}|\mu(\hat\bbeta,\bx)-\mu^*(\bx)|^2 \leq [16\epsilon_n^2]/C_F\asymp \epsilon_n^2\overline L^2 F_n^2.
\end{equation}
\end{theorem}

\section{Consistent Sparse DNN: Computation}
\label{Sectcomp}

The theoretical results established in previous sections
show that the Bayesian sparse DNN can be learned with a mixture Gaussian prior and, more importantly, the posterior inference is not necessarily directly drawn based on posterior samples, which avoids the convergence issue of the MCMC implementation for large complex models.
As shown in Theorems \ref{MAPthem1}, \ref{generalization1C} and \ref{generalization2C}, for the sparse BNN, a good local maximizer of the log-posterior distribution also guarantees  consistency of the network structure selection and asymptotic optimality of the network generalization performance. 
This local maximizer, in the spirit of condition C.3 and the conditions of Theorems \ref{generalization1C} and \ref{generalization2C}, is not necessarily a MAP estimate,  
as the factor 
$\det(-\frac{n}{2\pi}H_{n}(\hat{\bbeta}))^{\frac{1}{2}}$ can play an important role. 
In other words, an estimate of $\bbeta$ lies in a wide valley of the energy landscape is generally 
preferred. This is consistent with the view of many other authors, see e.g., \cite{widevalley1} and
\cite{widevalley2}, where different techniques have been developed to enhance convergence 
of SGD to a wide valley of the energy landscape.

Condition C.3 can be re-expressed as 
$\int_{\mR^{K_{n}}\setminus B_{\delta}(\hat{\bbeta})}
e^{nh_{n}(\bbeta))}d\bbeta=o(\det(-\frac{n}{2\pi}H_{n}(\hat{\bbeta}))^{-\frac{1}{2}}e^{nh_{n}(\hat{\bbeta})})$, which requires that $\hat\bbeta$ is a dominating mode of the posterior. 
 Based on this observation, we suggest to use the Bayesian  evidence \citep{liang2005evidence, mackay1992evidence} as the criterion  for eliciting estimates of $\bbeta$ produced by an optimization method in multiple runs with different initializations. 
The Bayesian evidence is calculated as  
$\det(-\frac{n}{2\pi}H_{n}(\hat{\bbeta}))^{-\frac{1}{2}}e^{nh_{n}(\hat{\bbeta})}$. 
Since Theorem \ref{Selectlem} ensures only consistency of structure selection but not consistency of parameter estimation, 
we suggest to refine its nonzero weights by a short 
optimization process after structure selection. The complete algorithm is summarized in Algorithm \ref{evidence}.

\begin{algorithm}[tb]
\caption{Sparse DNN Elicitation with Bayesian Evidence}
\label{evidence}
\begin{algorithmic}
\STATE Input: $T$---the number of independent tries in training the DNN, and the prior hyperparameters $\sigma_{0,n}$, $\sigma_{1,n}$, and $\lambda_n$.

  \FOR{$t=1,2, ..., T$}
   \STATE (i) {\it Initialization}: Randomly initialize the weights and biases, set $\bgamma_{i}$=1 for $i=1,2,\ldots,K_n$.
   \STATE (ii) {\it Optimization}: Run SGD to maximize $h_n(\bbeta)$ as defined in (\ref{log-posterioreq}). Denote the estimate of $\bbeta$ by $\hat{\bbeta}$.
   \STATE (iii) {\it Connection sparsification}: For each $i \in \{1,2,\ldots,K_n\}$, set $\bgamma_{i}=1$ if 
   $|\hat{\bbeta}_i|>\frac{\sqrt{2} \sigma_{0,n}\sigma_{1,n}}{\sqrt{\sigma_{1,n}^2-\sigma_{0,n}^2}} \sqrt{\log\left( \frac{1-\lambda_n}{\lambda_n}  \frac{\sigma_{1,n}}{\sigma_{0,n}} \right)}$ and 0 otherwise. 
  Denote the yielded sparse DNN structure by $\bgamma^t$,
   and set $\hat{\bbeta}_{\bgamma^t}=\hat{\bbeta} \circ \bgamma^t$, where $\circ$ denotes element-wise production.
  
   \STATE (iv) {\it Nonzero-weights refining}: Refine the nonzero weights of the sparsified DNN by maximizing 
    \begin{equation} \label{object2}
     h_{n}(\bbeta_{\bgamma^t})=\frac{1}{n}\sum_{i=1}^{n}\log(p(y_{i},\boldsymbol{x}_{i}|\bbeta_{\bgamma^t}))+\frac{1}{n}\log(\pi(\bbeta_{\bgamma^t})), 
     \end{equation}
     which can be accomplished by running SGD for a few epochs with the initial value $\hat{\bbeta}_{\bgamma^t}$. 
    Denote the resulting DNN model by $\tilde{\bbeta}_{\bgamma^t}$.
   \STATE (v) {\it Model evaluation}: Calculate the Bayesian Evidence: ${Evidence}^t = \det(-\frac{n}{2\pi}H_{n}(\tilde{\bbeta}_{\bgamma^t}))^{-\frac{1}{2}}e^{nh_{n}
   (\tilde{\bbeta}_{\bgamma^t})}$,   where $H_n(\bbeta_{\bgamma})=\frac{\partial^2 h_n(\bbeta_{\bgamma})}{\partial\bbeta_{\bgamma}
   \partial^T\bbeta_{\bgamma}}$ is the Hessian matrix.
   
\ENDFOR
\STATE Output $\tilde{\bbeta}_{\bgamma^t}$ with the largest Bayesian evidence.
\end{algorithmic}
\end{algorithm}

For a large-scale neural network, even if it is sparse, the number of nonzero elements 
 can easily exceed a few thousands or millions, see e.g. the networks 
considered in Section \ref{RealSection}. In this case, evaluation of the determinant of 
the Hessian matrix can be very time consuming. 
For this reason, we suggest to approximate 
the log(Bayesian evidence) by 
$nh_n(\hat{\bbeta_{\bgamma}}) -\frac{1}{2}|\bgamma|\log(n)$ with the detailed 
arguments given in Section 2.5 of the supplementary material. As explained there, if  the prior information imposed on the sparse DNNs is further ignored, then the sparse DNNs can be elicited by BIC.
  
 
 The main parameters for Algorithm \ref{evidence} are the prior hyperparameters 
 $\sigma_{0,n}$, $\sigma_{1,n}$, and $\lambda_n$.
Theorem \ref{2normal} provides theoretical suggestions for the choice of the prior-hyperparameters, see also the proof of Lemma \ref{priorjust} for a specific  setting for them. 
  Our theory allows $\sigma_{1,n}$ to grow with $n$ from the perspective of data fitting, but in our experience, the magnitude of weights tend to adversely affect the generalization ability of the network. For this reason, we usually set $\sigma_{1,n}$ to a relatively small number such as 0.01 or 0.02, and then
 tune the values of $\sigma_{0,n}$ and 
$\lambda_n$ for the network sparsity as well as the network approximation error. As a trade-off, the resulting 
network might be a little denser than the ideal one. 
If it is too dense to satisfy the sparse constraint given in Assumption A.2.2, one might increase the value 
of $\sigma_{0,n}$ and/or decrease the value of $\lambda_n$,
and rerun the algorithm to get a sparser structure. This process can be repeated until the constraint is satisfied.

 Algorithm \ref{evidence} employs SGD 
to optimize the log-posterior of the BNN. Since SGD generally converges to a local optimal solution, the multiple initialization method is used in order to find a local optimum close to the global one. It is interesting to note that SGD has some nice properties in non-convex optimization: It works on the convolved
(thus smoothed) version of the loss function 
\citep{Kleinberg2018} and tends to converge to flat local minimizers which are with very high probability also global minimizers \citep{zhang2018theory}. 
In this paper, we set the number of initializations to $T=10$ as default unless otherwise stated. 
We note that Algorithm \ref{evidence} is not very sensitive to the value of $T$, although a large value of $T$ can generally improve its performance.
 

For network weight initialization, we adopted the standard method, see \cite{glorot2010understanding} for tanh activation and \cite{he2015delving} for ReLU activation, which ensures that the variance of the gradient of each layer is of the same order at the beginning of the training process.

\section{Numerical Examples}

This section demonstrates the performance of the proposed algorithm on synthetic datasets and real datasets.\footnote{The code to reproduce the results of the experiments can be found at \url{https://github.com/sylydya/Consistent-Sparse-Deep-Learning-Theory-and-Computation}}

\subsection{Simulated Examples}
For all simulated examples, the covariates 
$x_1,x_2,\ldots,x_p$ were generated in the following procedure: (i) simulate $e,z_{1},\dots,z_{p_n}$ independently from the truncated standard normal distribution on the interval $[-10, 10]$; and 
(ii) set $x_{i}=\frac{e+z_{i}}{\sqrt{2}}$ for $i=1,2,\ldots,p_n$. In this way, all the covariates fall into a compact set and are mutually correlated with a
correlation coefficient of about 0.5. 
 We generated 10 datasets for each example. Each dataset consisted of $n=10,000$ training
samples, 1000 validation samples and 1000 test samples. 
For comparison, the sparse input neural network (Spinn) \citep{feng2017sparse},
dropout \citep{srivastava2014dropout} , and 
dynamic pruning with feedback (DPF) \citep{lin2020dynamic}  methods were also applied to these examples. 
Note that the validation samples were only used by Spinn, 
but not by the other methods. The performances of these methods in 
variable selection or connection selection were measured  
using the false selection rate (FSR) and the negative selection rate (NSR)\cite{ye2018variable}:
\[
FSR=\frac{\sum_{i=1}^{10}|\hat{S}_{i}\backslash S|}{\sum_{i=1}^{10}|\hat{S}_{i}|}, \quad \quad 
NSR=\frac{\sum_{i=1}^{10}|S\backslash\hat{S}_{i}|}{\sum_{i=1}^{10}|S|},
\]
where $S$ is the set of true variables/connections, $\hat{S}_{i}$ is the set
of selected variables/connections for dataset $i$, and $|\hat{S}_{i}|$ is the
size of $\hat{S}_{i}$. For the regression examples, the prediction and fitting performances 
of each method were measured by mean square prediction error (MSPE) and mean square fitting
error (MSFE), respectively; and for the classification examples, they were measured by prediction accuracy (PA) and fitting accuracy (FA), respectively.  To make each method to achieve the best or nearly best performance, we intentionally set the training and nonzero-weight refining process  excessively long.
In general, this is unnecessary. For example, for the 
read data example reported in Section \ref{RealSection}, the nonzero-weight refining process consisted of only one epoch.

\subsubsection{Network Structure Selection}
We generated 10 datasets from the following neural network model: 
\[
y = \tanh(2\tanh(2x_1 - x_2)) + 2\tanh(\tanh(x_3 - 2x_4) - \tanh(2x_5)) + +0x_{6}+\cdots+0x_{1000}+ \varepsilon,
\]
where $\varepsilon\sim N(0,1)$ and is independent of $x_i$'s. 
We fit the data using a neural network 
with structure 1000-5-3-1 and $\tanh$ as the activation function. For each dataset, we ran SGD for 80,000 iterations to train the neural network with a learning rate of $\epsilon_{t}={0.01}$. 
The subsample size was set to 500.
For the mixture Gaussian prior, we set $\sigma_{1,n}=0.01$, $\sigma_{0,n}= 0.0005$, and  $\lambda_{n}=0.00001$. The number of independent tries was set to $T=10$.  
After structure selection, the DNN was retrained using SGD for 40,000 iterations. Over the ten datasets, we got MSFE$=1.030$ (0.004) and MSPE$=1.041$ (0.014), where the numbers in the 
parentheses denote the standard deviation of the estimates. This indicates a good 
approximation of the neural network to the underlying true function. 
In terms of variable selection, we got perfect results with $FSR = 0$ and $NSR = 0$. 
In terms of structure selection, under the above setting, our method selected a little 
more connections with $FSR = 0.377$ and $NSR = 0$. The left panel of Figure \ref{Structure_Selection} shows a network structure 
selected for one dataset, which includes a few more 
connections than the true network. This ``redundant'' connection selection 
phenomenon is due to that $\sigma_{1,n}=0.01$ was too
small, which enforces more connections to be included in the network in order to compensate the effect of the shrunk true connection weights. 
However, in practice, such an under-biased setting of $\sigma_{1,n}$ is usually preferred 
from the perspective of neural network training and prediction, 
which effectively prevents the neural network to include large connection weights. 
It is known that including large connection weights in the 
neural network is likely to cause vanishing gradients in training as well as 
large error in prediction especially when the future observations are beyond the 
range of training samples. We also note that such a ``redundant'' connection selection 
phenomenon can be alleviated by performing another round of structure selection 
after retraining. In this case, the number of false connections included in the 
network and their effect 
on the objective function (\ref{object2}) are small, the true connections can 
be easily identified even with an under-biased value of $\sigma_{1,n}$.  The right panel of 
Figure \ref{Structure_Selection} shows the network structure selected after retraining, 
which indicates that the true network structure can be identified almost exactly. 
Summarizing over the networks selected for the 10 datasets, we had 
$FSR = 0.152$ and $NSR = 0$ after retraining; that is, on average, there were only 
about 1.5 more connections selected than the true network for each dataset. 
This result is remarkable!


For comparison, we have applied the the sparse input neural network (Spinn) method \citep{feng2017sparse} to this example, where Spinn was run with a LASSO penalty and a regularization parameter of $\lambda = 0.05$. 
We have tried $\lambda\in \{0.01,0.02,\dots,0.1 \}$ and found that $\lambda=0.05$ generally led to a better structure selection result. 
In terms of structure selection, Spinn got $FSR = 0.221$ and $NSR = 0.26$. 
Even with another round of structure selection after retraining, Spinn only got $FSR = 0.149$ and $NSR = 0.26$. It indicates that Spinn missed some true connections, which is inferior to the proposed method.



\begin{figure}
\centering
\includegraphics[scale=0.495]{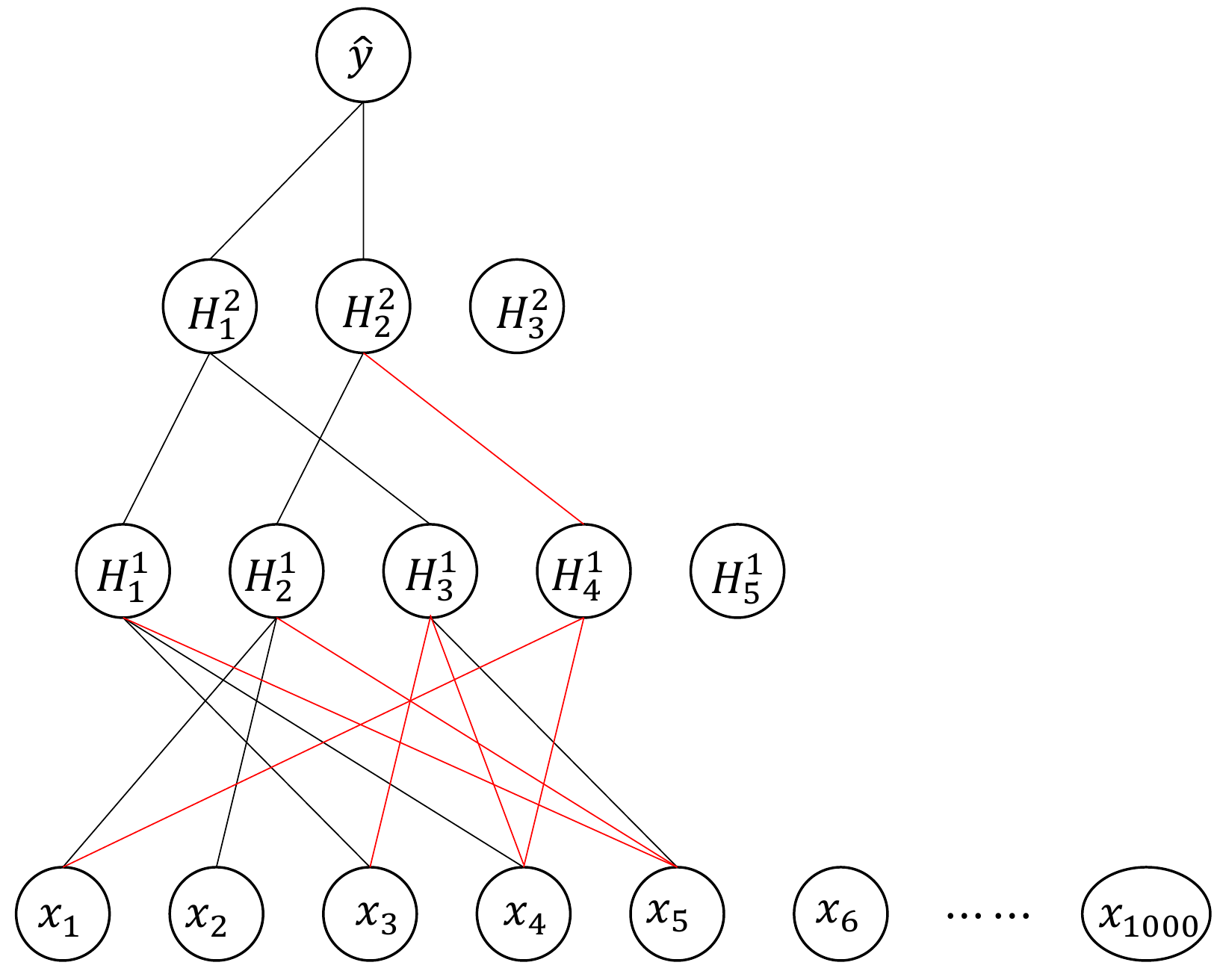}
\includegraphics[scale=0.495]{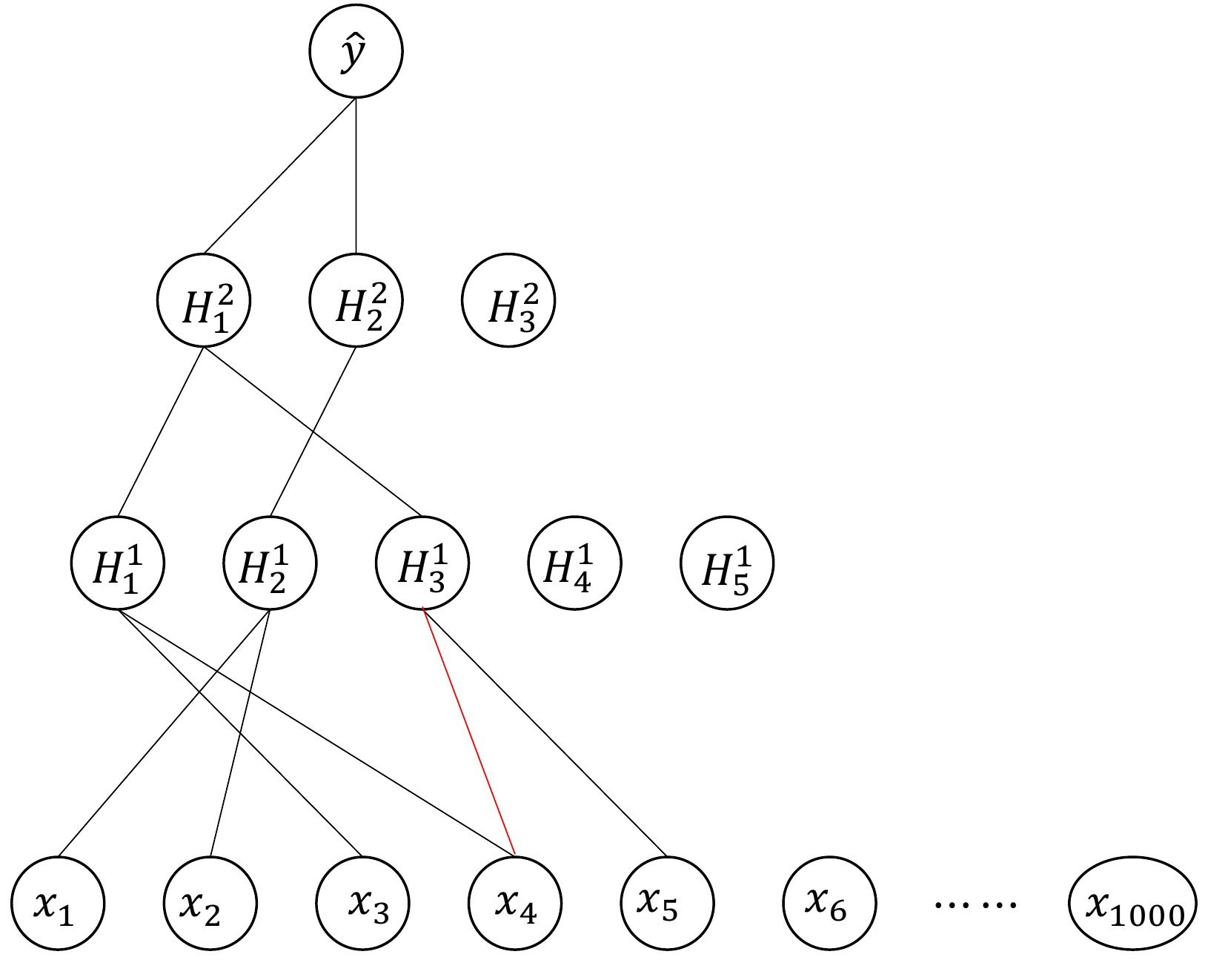}
\caption{The left and right panels show the network structures selected for one dataset 
after the stages of training and retraining, respectively, 
where the black lines show the connections that are selected and exist in the true model, 
and the red lines show the connections that are selected but do not exist 
in the true model. }
\label{Structure_Selection}
\end{figure}

\subsubsection{Nonlinear Regression} \label{regsect}

We generated 10 datasets from the following model:
\begin{equation} \label{nonlinear}
y=\frac{5x_{2}}{1+x_{1}^{2}}+5\sin(x_{3}x_{4})+2x_{5}+0x_{6}+\cdots+0x_{2000}+\varepsilon,
\end{equation}
where $\varepsilon\sim N(0,1)$ and is independent of $x_i$'s. 
We modeled the data by a 3-hidden layer neural network, which has 6, 4, and 3 
hidden units on the first, second and third hidden layers, respectively.  The tanh was used 
as the activation function. 
For each dataset, we ran SGD for 80,000 iterations to train the neural network with a learning rate  of $\epsilon_{t}={0.005}$. The subsample size was set to 500.
For the mixture Gaussian prior, we set $\sigma_{1,n}=0.01$, $\sigma_{0,n}= 0.0001$, and  $\lambda_{n}=0.00001$. The number of independent tries was set to $T=10$.  
After structure selection, the DNN was retrained using SGD for 40,000 iterations. 

For comparison, the Spinn,  dropout and DPF methods   
were also applied to this example with the same DNN structure and the same activation function. 
For Spinn, the regularization parameter for
the weights in the first layer was tuned from the set $\{0.01,0.02,\dots,0.1\}$.
For a  fair comparison, it
was also retrained for 40,000 iterations after structure selection as for the proposed method. 
For dropout, we set the dropout rate to be 0.2 for the first layer and 0.5 for the other layers. 
For DPF, we set the target pruning ratio to the ideal value 0.688\%, which is the ratio of the number of connections related to true variables and the total number of connections, i.e. $(12053 - 1995\times 6)/12053$ with 12053 being the total number of connections including the biases. The results were summarized in Table \ref{Simulation}.   

\begin{figure}
\centering
\includegraphics[scale=0.5]{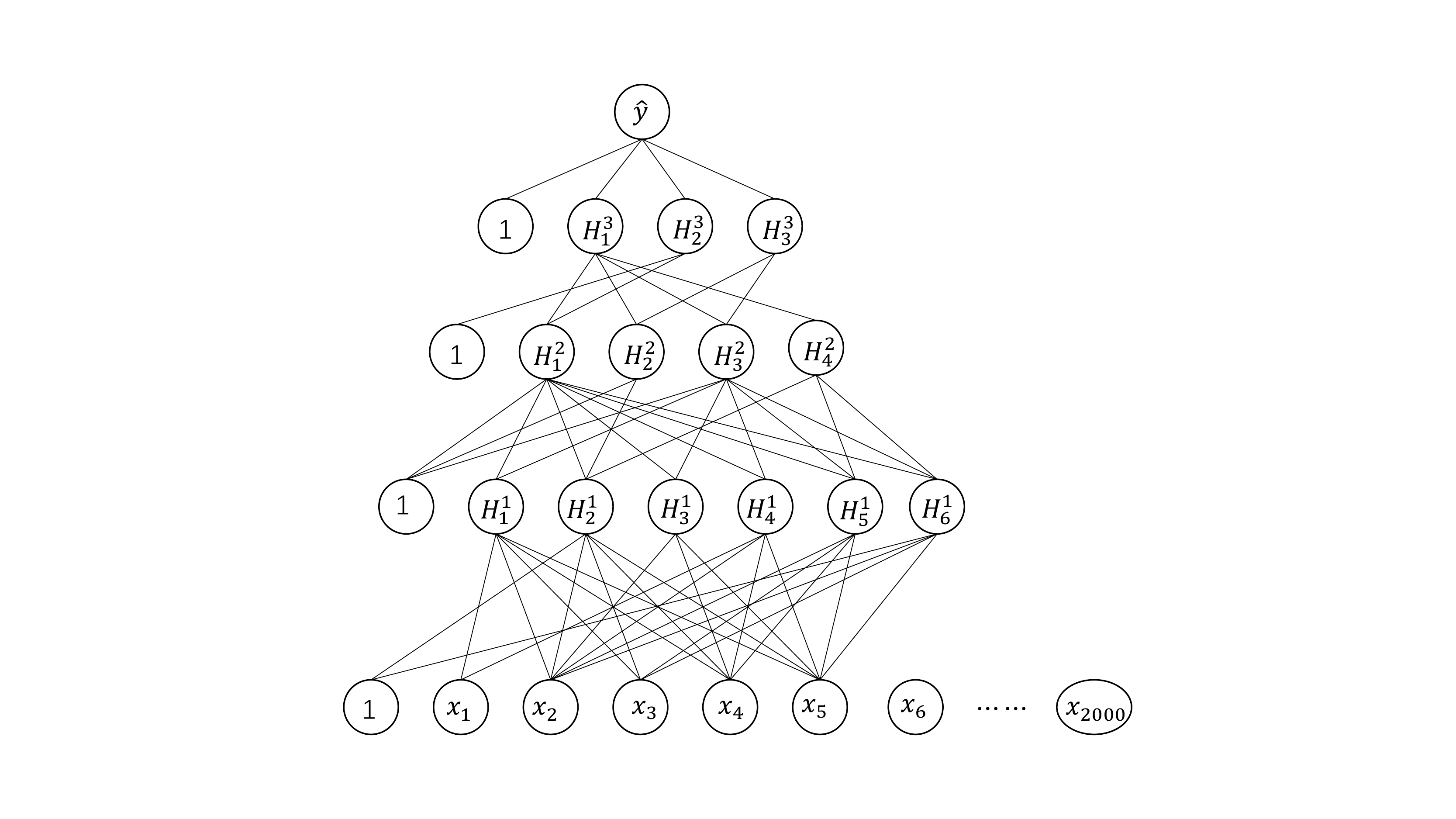}
\caption{A network structure selected by the Bayesian evidence method for a simulated 
 dataset from model (\ref{nonlinear}), where the node `1' denotes the bias term.}
\label{Regression_Structure}
\end{figure}

\begin{table}
\caption{Comparison of different methods for the simulated nonlinear regression example, where MSFE and MSPE were calculated by averaging over 10 datasets with the standard deviation given in the parentheses. }
\label{Simulation}
\begin{center}
\begin{tabular}{cccccccc}
 \toprule
Activation & Method & $|\hat S|$ & FSR & NSR & MSFE & MSPE\\
  \midrule
& BNN & {\bf 5}(0) & {\bf 0} & {\bf 0} & {\bf 2.372}(0.093) & {\bf 2.439}(0.132)\\
& Spinn & 36.1(15.816) & 0.861 & 0 & 3.090(0.194) & 3.250(0.196)\\
\raisebox{1.5ex}{Tanh} & DPF & 55.6(1.002) & 0.910 & 0 & 2.934(0.132) & 3.225(0.524) \\ 
& dropout & --- & --- & --- & 10.491(0.078) & 13.565(0.214) \\
\midrule
& BNN & {\bf 5}(0) & {\bf 0} & {\bf 0} & {\bf 2.659}(0.098) & {\bf 2.778}(0.111)\\
& Spinn & 136.3(46.102) & 0.963 & 0 & 3.858(0.243) & 4.352(0.171)\\
\raisebox{1.5ex}{ReLU} & DPF &    
67.8(1.606) & 0.934 & 0  &  5.893(0.619) & 6.252(0.480) \\
& dropout & --- & --- & --- & 17.279(0.571) & 18.630(0.559) \\ \bottomrule
\end{tabular}
\end{center}
\end{table}

Table \ref{Simulation} indicates that the proposed BNN method significantly outperforms the Spinn and DPF methods in both prediction and 
 variable selection. For this example, BNN
can correctly identify the 5 true variables of the nonlinear regression (\ref{nonlinear}),
while Spinn and DPF identified too many false variables. 
Figure \ref{Regression_Structure} shows the structure of a selected neural network by the BNN method. In terms of prediction, BNN, Spinn and DPF all significantly outperform the dropout method. In our experience, 
when irrelevant features are present in the data, learning a sparse DNN is always rewarded in prediction.

Figure \ref{Evidence_MSPE} explores the relationship between Bayesian evidence and prediction 
accuracy. Since we set the number of tries $T=10$ for each of the 10 datasets, 
there are a total of 100 
pairs of (Bayesian evidence, prediction error) 
shown in the plot. 
The plot shows a strong linear pattern that 
the prediction error of the sparse neural network decreases
as Bayesian evidence increases.
This justifies the rationale of 
Algorithm \ref{evidence}, where Bayesian evidence is employed for eliciting sparse neural 
network models learned by an optimization method in multiple runs with different initializations.  


\begin{figure}
\centering
\includegraphics[scale=0.525]{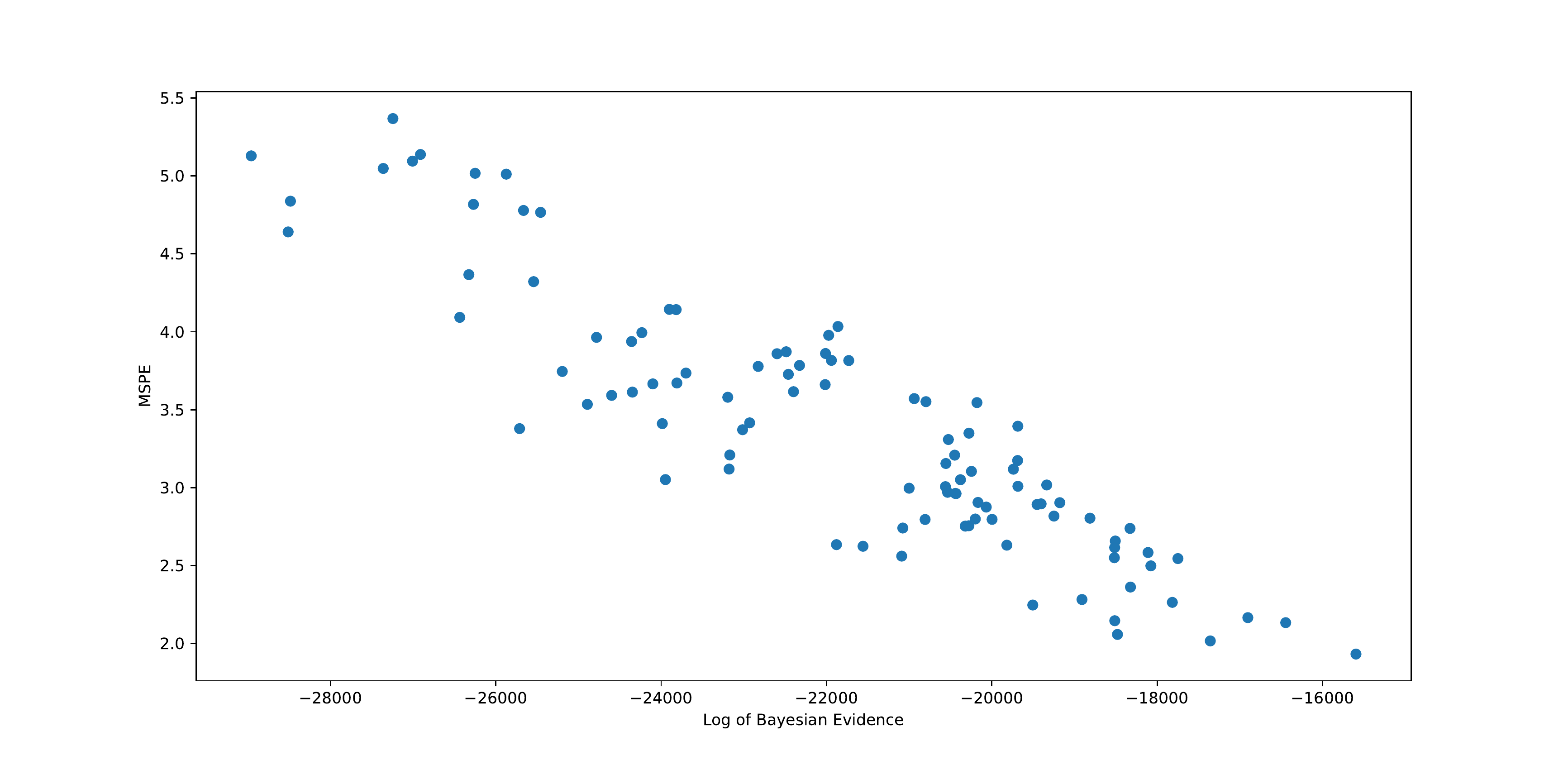}
\caption{Scatter plot of MSPE versus the logarithm of Bayesian Evidence, 
for which the fitted 
OLS regression line  is $y=-2.459\times10^{-4}x-1.993$ with 
 $R^2=0.8037$ and $p$-value = $2.2\times10^{-16}$.}
\label{Evidence_MSPE}
\end{figure}

For this example, we have also compared different methods with the ReLU activation. The same network structure and the same hyperparameter setting were used as in the experiments with the tanh activation.The results are also summarized in Table \ref{Simulation}, which indicate
that the proposed BNN method still significantly outperforms the competing ones.

\subsubsection{Nonlinear Classification}
We generated 10 datasets from the following nonlinear system:
\[
\begin{split}
    y=\begin{cases}
\begin{array}{c}
1\\
0
\end{array} & \begin{array}{c}
e^{x_{1}}+x_{2}^{2}+5\sin(x_{3}x_{4})-3+0x_{5}+\cdots+0x_{1000}>0,\\
 {\rm otherwise}.
\end{array}\end{cases}
\end{split}
\]
Each dataset consisted of half of the observations with the response $y=1$. We modeled the data by a 3-hidden layer logistic regression neural network, which had 6,4,3 hidden units on the first, second and third hidden layers, respectively. 
The tanh was used as the activation function.
The proposed method was compared with Spinn, dropout and DPF. For all the methods, 
the same hyperparameter values were used as in the nonlinear regression example of Section \ref{regsect}. 
The results are summarized in Table \ref{logistic_simulation}, which shows that 
the proposed method can identify the true variable for the nonlinear system 
and make more accurate prediction than Spinn, dropout and DPF. Compared to BNN and Spinn, DPF missed some true variables and produced a larger value of NSR.

\begin{table}
\caption{Comparison of different methods for the simulated classification example, where FA and PA were calculated by averaging over 10 datasets with the standard deviation given in the parentheses. }
\label{logistic_simulation}
\begin{center}
\begin{tabular}{ccccccc}
 \toprule
Method & $|\hat S|$ & FSR & NSR & FA & PA\\
  \midrule
BNN & {\bf 4}(0) & {\bf 0} & {\bf 0} & {\bf 0.8999}(0.0023) & {\bf 0.8958}(0.0039)\\
Spinn & 4.1(0.09) & 0.024 & 0 & 0.8628(0.0009) & 0.8606(0.0036)\\
DPF & 61.9(0.81) & 0.935 & 0.333 & 0.8920(0.0081) & 0.8697(0.0010)\\
dropout & --- & --- & --- & 0.4898(0.0076) & 0.4906(0.0071) \\
 \bottomrule
\end{tabular}
\end{center}
\end{table}

\subsection{Residual Network Compression} \label{RealSection}

This section assessed the performance of the proposed BNN method on network compression 
with CIFAR-10 \citep{krizhevsky2009learning} used as the illustrative dataset.
The CIFAR-10 dataset is a benchmark dataset for computer vision, 
which consists of 10 classes,  50,000 training images, and 10,000 testing images. 
 We modeled the data using both 
 ResNet20 and ResNet32 \citep{DNNRes2016} and 
 then pruned them to different sparsity levels.
We compared the proposed BNN method with DPF \citep{lin2020dynamic}, dynamic sparse reparameterization (DSR) \citep{mostafa2019parameter}, sparse momentum (SM) \citep{dettmers2019sparse}, and Variational Bayes(VB) \citep{Blundell2015}.
All experiments were implemented using Pytorch \citep{paszke2017automatic}. 

In all of our experiments, we followed the same training setup as used in \cite{lin2020dynamic}, i.e.  the model was trained using SGD with momentum for 300 epochs,  the data augmentation strategy \citep{zhong2017random} was employed, the mini-batch size was set to 128, the momentum parameter was set to 0.9, and the initial learning rate was set to 0.1. We divided the learning rate by 10 at epoch 150 and 225.
For the proposed method, we set the number of independent trials $T=10$, and used BIC to elicit sparse networks. In each trial, the mixture normal prior was imposed on the network weights after 150 epochs.
After pruning, the model was retrained for one epoch for refining the nonzero weights. For the mixture normal prior, we set $\sigma_{1,n}^2 = 0.02$ and tried different values for $\sigma_{0,n}$ and $\lambda_n$ to achieve different sparsity levels.
For ResNet-20, to achieve 10\% target sparsity, we set $\sigma_{0,n}^2 = 4e-5$ and $\lambda_n = 1e-6$; to achieve 20\% target sparsity, we set $\sigma_{0,n}^2 = 6e-6$ and $\lambda_n = 1e-7$. 
For ResNet-32, to achieve 5\% target sparsity,  we set $\sigma_{0,n}^2 = 6e-5$ and $\lambda_n = 1e-7$; to achieve 10\% target sparsity, we set $\sigma_{0,n}^2 = 2e-5$ and $\lambda_n = 1e-5$.

Following the experimental setup in \cite{lin2020dynamic}, all experiments were run for 3 times and the averaged test accuracy and standard deviation were reported. For the VB method, we followed \cite{Blundell2015} to impose a mixture Gaussian prior (the same prior as used in our method) on the connection weights, and employed a diagonal multivariate Gaussian distribution to approximate the posterior. 
 As in \cite{Blundell2015}, we ordered the connection weights in the signal-to-noise ratio $\frac{|\mu|}{\sigma}$, and identified a sparse structure by removing the weights with a low signal-to-noise ratio.
The results were summarized in Table \ref{CIFAR_revise}, where the results of other baseline methods were taken from \cite{lin2020dynamic}.
The comparison indicates that the proposed method is able to produce better prediction accuracy than the existing methods at about the same level of sparsity, and that the VB method is not very competitive in statistical inference although it is very attractive in computation.
Note that the proposed method provides a one-shot pruning strategy. As discussed in \cite{han2015learning}, it is expected that these results can be further improved with appropriately tuned hyperparameters and iterative pruning and retraining.


\begin{table}
\caption{Network compression for CIFAR-10 data, where the number in the parentheses denotes the standard deviation of the respective estimate.}
\label{CIFAR_revise}
\begin{center}
\begin{tabular}{ccccccccc} \toprule
     & \multicolumn{2}{c}{ResNet-20} & & 
     \multicolumn{2}{c}{ResNet-32} \\ \cline{2-3}\cline{5-6}
Method &  Pruning Ratio & Test Accuracy & & Pruning Ratio &  Test Accuracy \\
  \midrule
BNN & 19.673\%(0.054\%) & {\bf92.27(0.03)} & &  9.531\%(0.043\%) & {\bf92.74(0.07)}  \\ 
SM  & 20\% & 91.54(0.16) & &  10\% & 91.54(0.18)   \\
DSR & 20\% & 91.78(0.28) & &  10\% & 91.41(0.23) \\
DPF  & 20\% &  92.17(0.21) & &  10\% &  92.42(0.18) \\
VB  & 20\% &  90.20(0.04) & &  10\% &  90.11(0.06) \\
  \midrule
BNN&  9.546\%(0.029\%) & {\bf 91.27(0.05)} && 
4.783\%(0.013\%) & {\bf91.21(0.01)}   \\ 
SM&  10\% & 89.76(0.40) & & 5\% & 88.68(0.22)   \\
DSR&  10\% & 87.88(0.04) & & 5\% & 84.12(0.32)  \\
DPF&  10\% &  90.88(0.07) & & 5\% &  90.94(0.35)  \\

VB&  10\% & 89.33(0.16)  & & 5\% &  88.14(0.04)  \\
\bottomrule
\end{tabular}
\end{center}
\end{table}

The CIFAR-10 has been used  as a benchmark example in many DNN compression experiments. Other than the competing methods considered above, 
 the targeted dropout method \citep{gomez2018targeted} reported a 
Resnet32 model with 47K parameters (90\% sparsity) and the prediction accuracy 91.48\%. 
The Bayesian compression method with a group normal-Jeffreys prior (BC-GNJ) \citep{louizos2017bayesian} reported a VGG16 model, a very deep convolutional neural 
 network model  proposed by \cite{simonyan2014deep}, 
 with 9.2M parameters (93.3\% sparsity) and the prediction accuracy 91.4\%. A comparison with our results reported in
 Table \ref{CIFAR_revise} indicates again the 
 superiority of the proposed BNN method.
 
 Finally, we note that the proposed BNN method belongs to the class of pruning methods and it provides an effective way for learning  sparse  DNNs. Contemporary experience shows that directly training a sparse or small dense network from the start typically converges slower than training with a pruning method, see e.g. \cite{frankle2018lottery}\cite{ye2020good}. This issue can be illustrated using a network compression example.
Three experiments were conducted for a ResNe20 (with 10\% sparsity level) on the CIFAR 10 dataset:
(a) sparse BNN, i.e., running the proposed BNN method 
 with randomly initialized weights; 
(b) starting with sparse network, i.e., training the sparse network learned by the proposed BNN method but with the weights randomly reinitialized;  and 
(c) starting with small dense network, i.e., training a network whose number of parameters in each layer is about the same as that of the sparse network in experiment (b) and whose weights are randomly initialized. The experiment (a) consisted of 400 epochs,  where the last 100 epochs were used for refining the  nonzero-weights of the sparse network obtained at epoch 300 via {\it connection sparsification}. Both the experiments (b) and (c) consisted of 300 epochs. 
In each of the experiments, the learning rate was set in the standard scheme \citep{lin2020dynamic},
i.e., started with 0.1 and then decreased by a factor of 10 at epochs 150 and 225, respectively.  
For random initialization, we used the default method in PyTorch \citep{he2015delving}. For example, for a 2-D convolutional layer with $n_{in}$ input feature map channels, $n_{out}$ output feature map channels, and a convolutional kernel of size $w\times h$, the weights and bias of the layer were initialized by independent draws from the uniform distribution $Unif(-\frac{1}{\sqrt{n_{in}\times w \times h}}, \frac{1}{\sqrt{n_{in}\times w \times h}})$.

Figure \ref{dense_vs_sparse} shows the training and testing paths obtained in the three experiments.
 It indicates that the sparse neural network learned by the proposed method significantly outperforms the other two networks trained with randomly initialized weights.
  This result is consistent with the finding of 
  \cite{frankle2018lottery} that the architectures uncovered by pruning are harder to train from the start and they often reach lower accuracy than the original neural networks.

\begin{figure}
\centering
\includegraphics[scale=0.525]{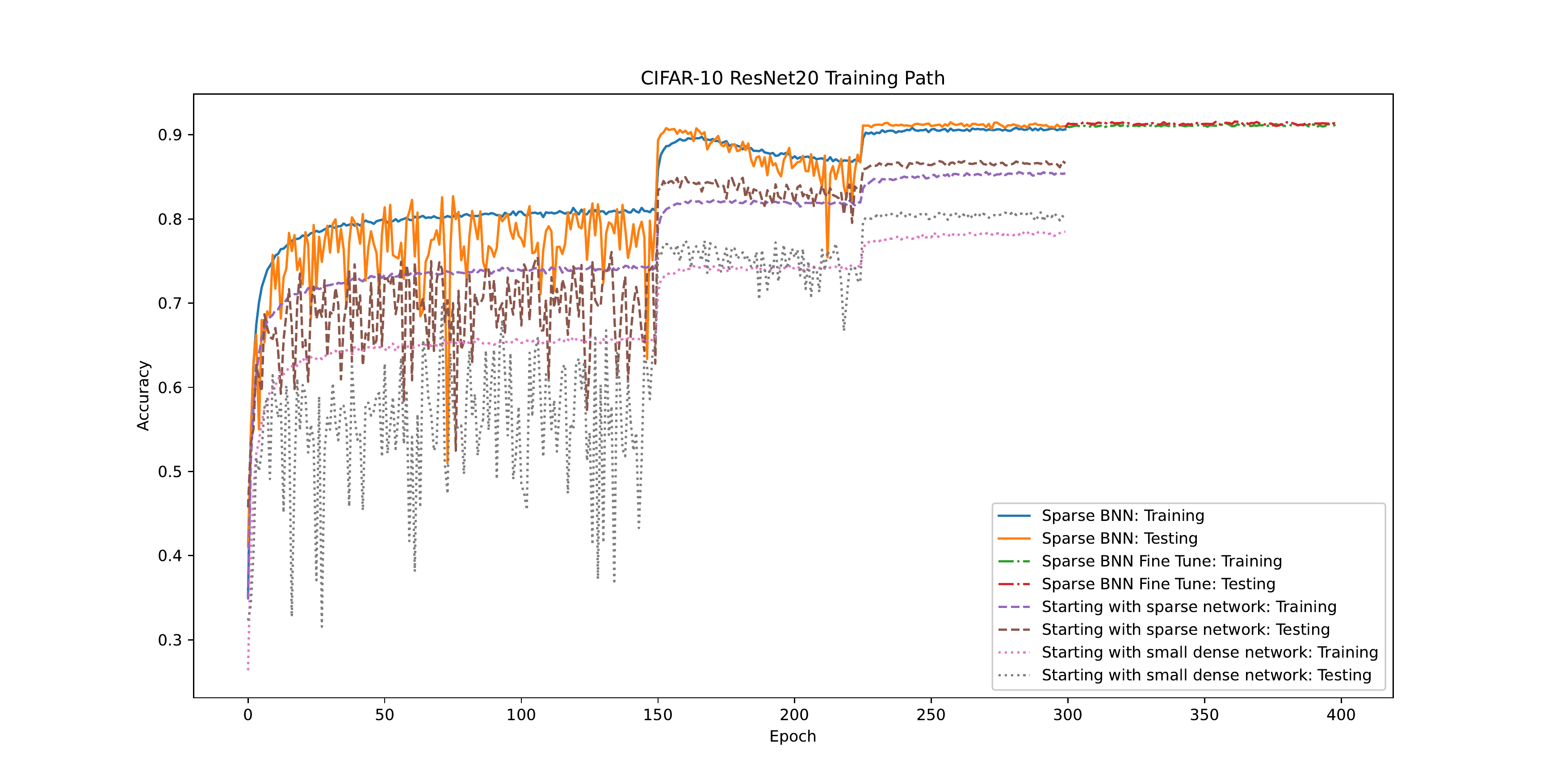}
\caption{Training and testing paths of a ResNet20 model (with 10\% sparsity level) on the CIFAR-10 dataset:
``Sparse BNN: Training'' is the path for fitting the full neural network in experiment (a) (for epochs 1-300);  
 ``Sparse BNN Fine Tune: Training'' is the path for 
refining the nonzero-weights of the sparse network obtained in experiment (a) via {\it connection sparsification} (for epochs 301-400); ``Starting with sparse network: Training'' is the path for fitting a sparse neural network in experiment (b);  and ``Starting with small dense network: Training'' is the path for fitting a small dense neural network in experiment (c). 
The curves for testing can be interpreted similarly. 
 }
\label{dense_vs_sparse}
\end{figure}

 Regarding this experiment, we have further two remarks.  First,  the nonzero weights refining step in Algorithm 1 can consist of a few epochs only.
 This step is mainly designed for a theoretical purpose, ensuring the followed evidence evaluation to be done on a local mode of the posterior. 
 In terms of sparse neural network learning, the mixture Gaussian prior plays a key role in sparsifying the neural network, while the nonzero weights refining step is not essential. As illustrated by Figure 4, this step did not significantly improve the training and testing errors of the sparse neural network. 
 Second, as mentioned previously, 
 the proposed BNN method falls into 
 the class of  pruning methods suggested by \cite{frankle2018lottery} for learning sparse DNNs. Compared to the existing pruning methods, the proposed method is more theoretically sound, which ensures the resulting sparse network to possess nice theoretical properties such as posterior consistency, variable  selection  consistency and asymptotically  optimal  generalization bounds.



\section{Discussion}

This paper provides a complete treatment for sparse DNNs in both theory and computation. The proposed method works like a frequentist method, but is justified under the Bayesian framework. With the proposed method, a sparse neural network with at most $O(n/\log(n))$ connections could be 
learned via sparsifying an over-parameterized one. 
Such a sparse neural network has nice theoretical properties, such as posterior consistency, variable selection consistency, and asymptotically optimal generalization bound.
 
 In computation, we proposed to use Bayesian evidence or BIC for eliciting 
 sparse DNN models learned by an optimization method in multiple runs  
 with different initializations. 
 Since conventional optimization methods such as SGD and Adam \citep{kingma2014adam} can be used to 
 train the DNNs, the proposed method is computationally more efficient than 
 the standard Bayesian method.  
 Our numerical results show that the proposed method  can perform very well in 
 large-scale network compression and high-dimensional nonlinear variable selection.
 The networks learned  by the proposed method tend to predict better than  the existing methods.

  Regarding the number of runs of the optimization method, i.e., the value of $T$, in Algorithm \ref{evidence}, we would note again that Algorithm \ref{evidence} is not very sensitive to it. For example, for the nonlinear regression example, Algorithm \ref{evidence} correctly identified the true variables in each of the 10 runs. 
For the CIFAR-10 example with ResNet20 and 10\% target sparsity level, Algorithm \ref{evidence} achieved the test accuracies in 10 runs: $91.09\%, 91.16\%, 91.17\%, 91.18\%, 91.18\%, 91.24\%$, $91.24\%, 91.24\%, 91.27\%, 91.31\%$ (in ascending order), where the worst one is still better than those achieved by the baseline methods. 

 
In this work, we choose the two-mixture Gaussian prior for the weights and biases of the DNN, 
mainly for the sake of computational convenience.
Other choices, such as two-mixture Laplace prior \citep{rovckova2018bayesian}, which
will lead to the same posterior contraction with an appropriate choice for the prior hyperparameters. To be more specific, Theorem S1 (in the supplementary material) establishes sufficient conditions that guarantee the posterior consistency, and any prior distribution satisfying the sufficient conditions can yield consistent posterior inferences for the DNN.

Beyond the absolutely continuous prior, the hierarchical prior used in \cite{liang2018bayesian} and \cite{PolsonR2018} can 
be adopted for DNNs. To be more precise, one can assume that
\begin{equation} \label{prioreq1}
\boldsymbol{\beta}_{\boldsymbol{\gamma}}\mid\boldsymbol{\gamma}\sim N(0,\sigma_{1,n}^{2}\boldsymbol{I}_{\left|\boldsymbol{\gamma}\right|\times\left|\boldsymbol{\gamma}\right|}),
\quad \bbeta_{\bgamma^c} = 0;
\end{equation}
\begin{equation} \label{prioreq2}
\pi(\boldsymbol{\gamma})\propto\lambda_{n}^{\left|\boldsymbol{\gamma}\right|}\left(1-\lambda_{n}\right)^{K_{n}-\left|\boldsymbol{\gamma}\right|}\boldsymbol{1}\left\{ 1\leq\left|\boldsymbol{\gamma}\right|\leq\bar{r}_{n},\boldsymbol{\gamma}\in\mathcal{G}\right\} ,
\end{equation}
 where $\bbeta_{\bgamma^c}$ is the complement of $\bbeta_{\bgamma}$, $\left|\boldsymbol{\gamma}\right|$ is the number of nonzero
elements of $\boldsymbol{\gamma}$, $\boldsymbol{I}_{\left|\boldsymbol{\gamma}\right|\times\left|\boldsymbol{\gamma}\right|}$
is a $\left|\boldsymbol{\gamma}\right|\times\left|\boldsymbol{\gamma}\right|$
identity matrix, $\bar{r}_{n}$ is the maximally allowed size of candidate networks, 
$\mathcal{G}$ is the set of valid DNNs, and 
the hyperparameter $\lambda_{n}$, as in (\ref{marprior}), 
can be read as an approximate prior probability for
each connection or bias to be included in the DNN.
 Under this prior, the product of the weight or bias and its indicator
follows a discrete spike-and-slab prior distribution, i.e.
\[
{\bw}_{ij}^{h}\bgamma_{ij}^{\boldsymbol{w}^{h}}|\bgamma_{ij}^{\boldsymbol{w}^{h}}\sim\bgamma_{ij}^{\boldsymbol{w}^{h}}N(0,\sigma_{1,n}^{2})+(1-\bgamma_{ij}^{\boldsymbol{w}^{h}})\delta_{0},\ \ 
{\bb}_{k}^{h}\bgamma_{k}^{\boldsymbol{b}^{h}}|\bgamma_{k}^{\boldsymbol{b}^{h}}\sim\bgamma_{k}^{\boldsymbol{b}^{h}}N(0,\sigma_{1,n}^{2})+(1-\bgamma_{k}^{\boldsymbol{b}^{h}})\delta_{0},
\]
where  $\delta_0$ denotes the Dirac delta function. 
Under this hierarchical prior, it is not difficult to show that the posterior consistency and structure selection 
consistency theory developed in this paper still hold. 
However, from the computational perspective, 
the hierarchical prior might be inferior to the mixture Gaussian prior adopted in the paper, 
as the posterior $\pi(\bbeta_{\bgamma}, \bgamma|D_n)$ is hard to be optimized or simulated from.
It is known that directly simulating 
from $\pi(\bbeta_{\bgamma}, \bgamma|D_n)$ using an acceptance-rejection based MCMC algorithm
can be time consuming. A feasible way is to
formulate the prior of $\bbeta_{\bgamma}$ as $\bbeta_{\bgamma}=\btheta\otimes\bgamma$, where $\btheta\sim N(0,\sigma^2_{1,n}I_{H_n\times H_n})$  can be viewed as a latent variable and 
$\otimes$ denotes entry-wise product. Then one can first simulate from the marginal posterior 
$\pi(\btheta|D_n)$ using a stochastic gradient MCMC algorithm and then make inference of 
the network structure based on the conditional posterior $\pi(\bgamma|\btheta,D_n)$. 
We note that the gradient $\nabla_{\btheta}\log\pi(\btheta|D^{n})$ can be approximated 
based on the following identity developed in  \cite{Liang2019eSGLD}, 
\[
\nabla_{\btheta}\log\pi(\btheta|D^{n})=\sum_{\bgamma}\pi(\bgamma|\btheta,D^{n})\nabla_{\btheta}\log\pi(\btheta|\bgamma,D^{n}),
\]
where $D_n$ can be replaced by a dataset duplicated with mini-batch samples 
if the subsampling strategy is used to accelerate the simulation. 
This identity greatly facilitates the simulations for the dimension jumping problems,
which requires only some samples to be drawn from the conditional posterior 
$\pi(\bgamma|\btheta,D^{n})$ for approximating the gradient $\nabla_{\btheta}\log\pi(\btheta|D^{n})$
at each iteration. A further exploration of this discrete prior for its use in deep learning is 
of great interest, although there are some difficulties needing to be addressed 
in computation.





\section*{Acknowledgement}
Liang's research was supported in part by the grants  DMS-2015498, R01-GM117597 and R01-GM126089.
Song's research was supported in part by the grant DMS-1811812. 
The authors thank the editor, associate editors and two referees for their constructive comments which have led to significant improvement of this paper.

\newpage

\section*{Supplementary Material}

\setcounter{table}{0}
\renewcommand{\thetable}{S\arabic{table}}
\setcounter{figure}{0}
\renewcommand{\thefigure}{S\arabic{figure}}
\setcounter{equation}{0}
\renewcommand{\theequation}{S\arabic{equation}}
\setcounter{algorithm}{0}
\renewcommand{\thealgorithm}{S\arabic{algorithm}}
\setcounter{lemma}{0}
\renewcommand{\thelemma}{S\arabic{lemma}}
\setcounter{theorem}{0}
\renewcommand{\thetheorem}{S\arabic{theorem}}
\setcounter{remark}{0}
\renewcommand{\theremark}{S\arabic{remark}}
\setcounter{section}{0}
\renewcommand{\thesection}{S\arabic{section}}

This material is organized as follows. Section S1 gives the proofs on posterior consistency, Section S2 gives the proofs on structure selection consistency, Section S3 gives the proofs on generalization bounds, and Section S4 gives some 
mathematical facts of the sparse DNN. 

\section{Proofs on Posterior Consistency}

\subsection{Basic Formulas of Bayesian Neural Networks}

\paragraph{Normal Regression.} 
 Let $p_{\mu}$ denote the density of $N(\mu, \sigma^2)$ 
 where $\sigma^2$ is a known constant, and let $p_{\bbeta}$ denote 
 the density of $N(\mu(\bbeta,\bx), \sigma^2)$.
 Extension to the case $\sigma^2$ is unknown is simple by following the arguments given in \cite{Jiang2007}.  
 In this case, an inverse gamma prior can be assumed for $\sigma^2$ as suggested by \cite{Jiang2007}. 
 Define the Kullback-Leibler divergence as $d_0(p, p^*)=\int p^* \log(p^*/p)$ for two densities $p$ and $p^*$.  
 Define a distance $d_t(p,p^*)=t^{-1} (\int p^* (p^*/p)^t-1)$ for any $t>0$, which 
 decreases to $d_0$ as $t$ decreases toward 0.
 A straightforward calculation shows 
\begin{equation} \label{normaleq1}
 d_1(p_{\mu_1},p_{\mu_2})= \int p_{\mu_1}(p_{\mu_1}/p_{\mu_1})-1 = \exp\left(\frac{1}{\sigma^2} (\mu_{2}-\mu_{1})^2 \right)-1 
 =\frac{1}{\sigma^2} (\mu_{2}-\mu_{1})^2+o((\mu_{2}-\mu_{1})^3),
\end{equation}
\begin{equation} \label{normaleq2}
 d_0(p_1,p_2)  = \frac{1}{2 \sigma^2} (\mu_1-\mu_2)^2.
 \end{equation} 
 
 \paragraph{Logistic Regression.}  
  Let $p_{\mu}$ denote the probability mass function with the success probability 
  given by $1/(1+e^{-\mu})$. 
  Similarly, we define $p_{\bbeta}$ as the logistic regression density for a binary classification DNN with parameter $\bbeta$. For logistic regression, we have  
 \[
 d_1(p_{\mu_1}, p_{\mu_2})=\int p_{\mu_2}(p_{\mu_2}/p_{\mu_1}) -1= \frac{e^{2 \mu_{2}-\mu_{1}} +e^{\mu_{1}}  
  -2 e^{\mu_{2}} }{(1+e^{\mu_{2}})^2},
 \]
 which, by the mean value theorem, can be written as  
 \[
 d_1(p_{\mu_1}, p_{\mu_2})  =\frac{e^{\mu'}- e^{2 \mu_2-\mu'}}{(1+e^{\mu_2})^2} (\mu_{\mu_1}-\mu_2)  
 = \frac{e^{\mu'} (1-e^{2\mu_2-2\mu'})}{(1+e^{\mu_2})^2} (\mu_{1}-\mu_{2}), 
 \]
 where $\mu'$ denotes an intermediate point between $\mu_{1}$ and $\mu_{2}$, and thus 
 $|\mu'-\mu_2| \leq |\mu_{1}-\mu_2|$.
 Further, by Taylor expansion, we have 
 \[
  e^{\mu'} =e^{\mu_2}[ 1+ (\mu'-\mu_2)+O((\mu'-\mu_2)^2) ], \quad 
  e^{2\mu_2-2\mu'} =1+2(\mu_2-\mu')+O((\mu_2-\mu')^2). 
 \]
 Therefore,
 \begin{equation} \label{logisticeq1}
 d_1(p_{\mu_1}, p_{\mu_2}) \leq  \frac{e^{\mu_2}}{(1+e^{\mu_2})^2}  
 \left[ 2 |\mu_2-\mu'| +O((\mu_2-\mu')^2) \right] |\mu_{1}-\mu_2|
 \leq \frac{1}{2} (\mu_{1}-\mu_2)^2+ O((\mu_{1}-\mu_2)^3),
 \end{equation}
 and
\[
 d_0(p_u, p_v)= \int p_v (\log p_v -\log p_u) v_y(dy) = \log(1+e^{\mu_u})-\log(1+e^{\mu_v}) +
  \frac{e^{\mu_v}}{1+e^{\mu_v}} (\mu_v-\mu_u). 
\]
 By the mean value theorem, we have 
 \begin{equation} \label{logisticeq2}
 d_0(p_u, p_v) =\frac{e^{\mu'}}{1+e^{\mu'}} (\mu_u-\mu_v)+
   \frac{e^{\mu_v}}{1+e^{\mu_v}} (\mu_v-\mu_u) = [ \frac{e^{\mu_v}}{1+e^{\mu_v}} - 
  \frac{e^{\mu'}}{1+e^{\mu'}}] (\mu_v-\mu_u),
 \end{equation}
 where $\mu'$ denotes an intermediate point between $\mu_u$ and $\mu_v$. 

\subsection{Posterior Consistency of General Statistical Models} \label{secA2}

We first introduce a lemma concerning posterior consistency of  
 general statistical models. This lemma has been proved in \cite{Jiang2007}. 
Let $\mP_n$ denote a sequence of sets of probability densities, let $\mP_n^c$ denote
the complement of $\mP_n$, and let $\epsilon_n$ denote a sequence of positive numbers. 
Let $N(\epsilon_n,\mP_n)$ be the minimum number of Hellinger balls of radius $\epsilon_n$ 
 that are needed to cover $\mP_n$, i.e., $N(\epsilon_n, \mP_n)$ is the minimum of all $k$'s
 such that there exist sets $S_j=\{p: d(p,p_j) \leq \epsilon_n\}$, $j=1,\ldots,k$, with 
 $\mP_n \subset \cup_{j=1}^k S_j$ holding, where $d(p,q)=\sqrt{ \int (\sqrt{p}-\sqrt{q})^2 }$ denotes the 
 Hellinger distance between the two densities $p$ and $q$. 

Let $D_n=(z^{(1)}, \ldots, z^{(n)})$ denote the dataset, where the observations $z^{(1)}, \ldots, 
 z^{(n)}$ are iid with the true density $p^*$. The dimension of $z^{(1)}$ and $p^*$ can depend on $n$.
 Define $\pi(\cdot)$ as the prior density, and $\pi(\cdot|D_n)$ as the posterior.
 Define $\hat{\pi}(\epsilon)=\pi[d(p,p^*)>\epsilon|D_n]$ for each $\epsilon>0$. 
 Define the KL divergence 
 as $d_0(p,p^*)=\int p^* \log(p^*/p)$. Define 
 $d_t(p,p^*)=t^{-1} (\int p^*(p^*/p)^t -1)$ for any $t>0$, which decreases to $d_0$ as $t$ 
 decreases toward 0.
 Let $P^*$ and $E^*$ denote the probability measure and expectation for the data $D_n$, respectively.
 Define the conditions:

 \begin{itemize}
 \item[(a)] $\log N(\epsilon_n, \mP_n) \leq n \epsilon_n^2$ for all sufficiently large $n$;
 \item[(b)] $\pi(\mP_n^c) \leq e^{-bn \epsilon_n^2}$ for all sufficiently large $n$;
 \item[(c)] $\pi[p: d_t(p,p^*) \leq b'\epsilon_n^2] \geq e^{- b'n\epsilon_n^2}$ for all sufficiently
            large $n$ and some $t>0$,
\end{itemize}
where $2>b>2b'>0$ are positive constants.
The following lemma is due to the same argument of \citet[][Proposition 1]{Jiang2007}.
\begin{lemma} \label{lem0}
 Under the conditions (a), (b) and (c) (for some $t>0$), given sufficiently large $n$, we have 
 \begin{itemize}
 \item[(i)] $P^*\left[\hat{\pi}(4 \epsilon_n) \geq 2 e^{-0.5n \epsilon_n^2 \min\{1,2-x,b-x,t(x-2b')\} } \right] \leq 
      2e^{-0.5n \epsilon_n^2 \min\{1,2-x,b-x,t(x-2b')\} }$,
\item[(ii)] $E^* \hat{\pi}(4 \epsilon_n) \leq 4 e^{-n \epsilon_n^2 \min\{1,2-x,b-x,t(x-2b')\}}$.
\end{itemize}
for any $2b'<x<b$.
\end{lemma}

\subsection{General Shrinkage Prior Settings for Deep Neural Networks}

Let $\bbeta$ denote the vector of parameters, including the weights of connections and the biases of the hidden and output units,  of a deep neural network.
Consider a general prior setting that all entries of $\bbeta$ are subject to independent continuous prior $\pi_b$, i.e., $\pi(\bbeta)=\prod_{j=1}^{K_n}\pi_b(\beta_j)$.
Theorem \ref{thm} provides a sufficient condition for posterior consistency.

\begin{theorem}[Posterior consistency] \label{thm}
Assume the conditions A.1, A.2 and A.3 hold, if the prior $\pi(\bbeta)$ satisfies that 
 \begin{align}
  &\log(1/\underline \pi_b)=O(H_n\log n+\log\overline L),\label{c1}\\
  &\pi_b\{[-\eta_n, \eta_n]\}\geq 1-\frac{1}{K_n}\exp\{-\tau[H_n\log n+\log\overline L+\log p_n]\} \mbox{ and } \pi_b\{[-\eta_n', \eta_n']\}\geq 1-\frac{1}{K_n}, \label{c3}\\
  & -\log[K_n\pi_b(|\beta_j|>M_n)] \succ n\epsilon_n^2,\label{c2} 
 \end{align}
 for some $\tau>0$,
 where $\eta_n< 1/\{\sqrt{n}K_n(n/H_n)^{H_n}(c_0M_n)^{H_n}\}$, 
 $\eta_n'< 1/\{\sqrt{n}K_n(r_n/H_n)^{H_n}(c_0E_n)^{H_n}\}$ with some $c_0>1$, 
 $\underline\pi_b$ is the minimal density value of $\pi_b$ within interval $[-E_n-1, E_n+1]$,  and $M_n$ is some sequence satisfying $\log(M_n)=O(\log(n))$.
Then, there exists a sequence $\epsilon_n$, satisfying $n\epsilon_n^2\asymp r_nH_n\log n+r_n\log\overline L+ s_n\log p_n+n\varpi_n^2$ and $\epsilon_n\prec 1$,
such that 
\begin{equation}\label{post}
\begin{split}
 & P^*\left\{ \pi[d(p_{\bbeta},p_{\mu^*}) > 4 \epsilon_n |D_n] \geq 2 e^{-n c\epsilon_n^2} \right\} 
  \leq 2 e^{-cn \epsilon_n^2},\\
 &  E_{D_n}^* \pi[d(p_{\bbeta},p_{\mu^*}) > 4 \epsilon_n | D_n] \leq 4 e^{-2cn \epsilon_n^2}.
 \end{split}
\end{equation}
for some $c>0$.
\end{theorem}

To prove Theorem \ref{thm}, we first introduce a useful Lemma:
\begin{lemma}[Theorem 1 of \cite{ZubkovS2013}]\label{lemmad}
        Let $X\sim \mbox{B}(n,v)$ be a Binomial random variable. For any $1<k<n-1$,
	\[ 
	Pr(X\geq k+1)\leq 1- \Phi(\mbox{sign}(k-nv)\{2nH(v, k/n)\}^{1/2}),
	\]
	where $\Phi$ is the cumulative distribution function (CDF) of the standard Gaussian distribution and
	$H(v, k/n)= (k/n)\log(k/nv)+(1-k/n)\log[(1-k/n)/(1-v)]$.
\end{lemma}

{\noindent \bf Proof of Theorem \ref{thm}}

Theorem \ref{thm} can be proved using Lemma \ref{lem0}, so it suffices to verify conditions (a)-(c) given in Section \ref{secA2}.

\noindent {\it Checking condition (c) for $t=1$:}

Consider the set $A=\{\bbeta:\max_{j\in\bgamma^*}\|\beta_j-\beta_j^*\|_\infty\leq \omega_n, \max_{j\notin\bgamma^*}\|\beta_j-\beta_j^*\|_\infty\leq \omega_n'\}$, where  $\omega_n={c_1\epsilon_n}/[H_n(r_n/H_n)^{H_n}(c_0E_n)^{H_n}]$ and $\omega_n'={c_1\epsilon_n}/[K_n(r_n/H_n)^{H_n}(c_0E_n)^{H_n}]$ for some constant $c_1>0$ and $c_0>1$. If $\bbeta\in A$, then by Lemma \ref{dnn2}, we have
$|\mu(\bbeta,\bx)-\mu(\bbeta^*,\bx)| \leq 3{c_1\epsilon_n}$. 
By condition A.2.1, $|\mu(\bbeta,\bx)-\mu^*(\bx)| \leq 3{c_1\epsilon_n}+\varpi_n$.
Combining it with
(\ref{normaleq1})--(\ref{logisticeq2}), for both normal and logistic models, we have
\[
d_1(p_{\bbeta},p_{\mu^*})\leq C(1+o(1)) E_x(\mu(\bbeta,\bx)-\mu^*(\bx))^2\leq C(1+o(1))(3{c_1\epsilon_n}+\varpi_n)^2, \quad\mbox{ if } \bbeta\in A,
\]
for some constant $C$.
Thus for any small $b'>0$,  condition (c) holds as long as that $c_1$ is sufficiently small, $n\epsilon_n^2\geq M_0 n\varpi_n^2$ for large $M_0$, and the prior satisfies
$-\log\pi(A)\leq b'{n\epsilon_n^2}$.

Since
$ \pi(A)\geq (2\underline \pi_b \omega_n)^{r_n} \times \pi(\{\max_{j\notin\bgamma^*}\|\beta_j\|\leq \omega_n'\})$,
$\pi_b([-\omega_n',\omega_n'])\geq 1-1/K_n$ (due to the fact $\omega_n'\gg\eta_n'$),
and $\log (1/\omega_n)\asymp\log(1/\epsilon_n)+H_n\log E_n+
H_n\log(r_n/H_n)+\mbox{constant}=O(H_n\log n)$ (note that $\log(1/\epsilon_n)=O(\log n)$),
the above requirement holds when $n\epsilon_n^2\geq M_0r_nH_n\log n$ for some sufficiently large constant $M_0$.

\noindent {\it Checking condition (a):}

Let $\mP_n$ denote the set of all DNN models whose weight parameter $\bbeta$ satisfies that 
\begin{equation}\label{bset}
\bbeta\in B_n=\{|\beta_j|\leq M_n, \bgamma_\bbeta=\{i: |\beta_i|\geq \delta_n'\} \mbox{ satisfies }|\bgamma_\bbeta|\leq k_nr_n \mbox{ and }
|\bgamma_\bbeta|_{in}\leq k_n's_n\},
\end{equation}
where $|\bgamma|_{in}$ denotes the input dimension of sparse network $\bgamma$, $k_n(\leq n/r_n)$ and $k_n'(\leq n/s_n)$ will be specified later, and $\delta_n=c_1\epsilon_n/[H_n(k_nr_n/H_n)^{H_n}(c_0M_n)^{H_n}]$ and $\delta_n'=c_1\epsilon_n/[K_n(k_nr_n/H_n)^{H_n}(c_0M_n)^{H_n}]$ for some constant $c_1>0$ and $c_0>1$.
Consider two parameter vectors $\bbeta^u$ and $\bbeta^v$ in set $B_n$,
such that there exists a model $\bgamma$ with $|\bgamma|\leq k_nr_n$ and $|\bgamma|_{in}\leq k_n's_n$, and 
$|\beta^u_j-\beta^v_j|\leq \delta_n $ for all $j\in \bgamma$, $\max(|\beta_j^u|,|\beta_j^v|)\leq \delta_n'$ for all $j \notin\bgamma$. Hence, by Lemma \ref{dnn2},
we have that $|\mu(\bbeta^u,\bx)-\mu(\bbeta^v,\bx)|^2 \leq 9c_1^2\epsilon_n^2$, and furthermore, due to (\ref{normaleq1})-(\ref{logisticeq2}),  
we can easily derive that
\[
d(p_{\bbeta^{u}},p_{\bbeta^{v}})\leq \sqrt{d_0(p_{\bbeta^{u}},p_{\bbeta^{v}})}\leq \sqrt{(9+o(1))c_1^2C\epsilon_n^2}\leq \epsilon_n,
\]
for some $C$, given a sufficiently small $c_1$. On the other hand, if some $\bbeta^u\in B_n$ and its connections whose magnitudes are larger than $\delta_n'$ don't form a valid network, then by Lemma \ref{dnn2} and (\ref{normaleq1})-(\ref{logisticeq2}), we also have that
$d(p_{\bbeta^{u}},p_{\bbeta^{o}})\leq\epsilon_n$,
where $\bbeta^o=0$ denotes a empty output network.

Given the above results, one can bound the packing number $N(\mP_n,\epsilon_n)$ by $
\sum_{j=1}^{k_nr_n}\mX^{j}_{H_n}\left(\frac{2M_n}{\delta_n}\right)^{j}$, where 
$\mX^{j}_{H_n}$ denotes the number of all valid networks who has exact $j$ connection and has no more than $k_n's_n$ inputs. Since
$\log\mX^{j}_{H_n}\leq k_n's_n\log p_n+ j\log(k_n's_nL_1+H_n\overline L^2)$, 
\[\begin{split}
\log N(\mP_n,\epsilon_n)&\leq \log k_nr_n+ k_nr_n\log H_n+2k_nr_n\log(\overline L+k_n's_n) 
+k_n's_n\log p_n\\
&+k_n r_n\log\frac{2M_nH_n(k_nr_n/H_n)^{H_n}M_n^{H_n}}{c_1 \epsilon_n}\\
& = k_nr_n*O\{H_n\log n+\log\overline L+\mbox{constant}
\}+k_n's_n\log p_n
\end{split}
\]
where the second inequality is due to $\log M_n=O(\log n)$, $k_nr_n\leq n$ and $k_n's_n\leq n$.
We can choose $k_n$ and $k_n'$ such that 
$k_nr_n\{H_n\log n+\log\overline L \}\asymp k_n's_n\asymp n\epsilon_n^2$ and $\log N(\mP_n,\epsilon_n)\leq n\epsilon_n^2$.

\noindent {\it Checking condition (b):}

$\pi(\mP_n^c) \leq Pr(\mbox{Binomial}(K_n,v_n)>k_nr_n)+K_n\pi_b(|\beta_j|>M_n)+Pr(|\bgamma_\bbeta|_{in}\geq k_n's_n)$,
where $v_n = 1-\pi_b([-\delta_n',\delta_n'])$.
By the condition of $\pi_b$ and the fact that
$\delta_n'\gg\eta_n$, $v_n \leq \exp\{-\tau[H_n\log n+\log\overline L+\log p_n)]-\log K_n\}$ for some positive constant $\tau$. Hence, by Lemma \ref{lemmad},
$
-\log Pr(\mbox{Binomial}(K_n,v_n)>k_nr_n)\approx \tau k_n r_n[H_n\log n+\log \overline L+ \log p_n]\gtrsim n\epsilon_n^2$ due to the choice of $k_n$,
and $
-\log Pr(|\bgamma_\bbeta|_{in}\geq k_n's_n)\approx  k_n' s_n[\tau (H_n\log n+\log\overline L+\log p_n)+\log (K_n/L_1p_n)]\gtrsim n\epsilon_n^2$ due to the choice of $k_n'$.
Thus, condition (b) holds as well.

\subsection{Proof of Theorem 2.1}

\begin{proof}
It suffices to verify the conditions listed in Theorem \ref{thm}.
Let $M_n=\max(\sqrt{2n}\sigma_{1,n}, E_n)$.
Condition (\ref{c1}) is due to $E_n^2/2\sigma_{1,n}^2+\log \sigma_{1,n}^2=O[H_n\log n+\log\overline L]$;
Condition (\ref{c3}) can be verified by 
$\lambda_n=1/\{K_n [n^{H_n}(\overline Lp_n)]^{\tau'}\}$ and
$\sigma_{0,n}\prec 1/\{\sqrt{n}K_n(n/H_n)^{H_n}(c_0M_n)^{H_n}\}$;
Condition (\ref{c2}) can be verified by $M_n\geq 2n\sigma_{0,n}^2$ and 
$\tau[H_n\log n+\log \overline L+\log p_n]+M_n^2/2\sigma_{1,n}^2\geq n$.

\end{proof}

\section{Proofs on Structure Selection Consistency}

\subsection{Proof of Theorem 2.2}
\begin{proof}
 \begin{equation}
  \begin{split}
   &\max|q_i-e_{i|\nu(\bgamma^*,\bbeta^*)}|
   \leq \max \int \sum_{\bgamma} |e_{i|\nu(\bgamma,\bbeta)}-e_{i|\nu(\bgamma^*,\bbeta^*)}|\pi(\bgamma|\bbeta,D_n)\pi(\bbeta|D_n) d\bbeta\\
   =& \max\int_{A(4\epsilon_n)} \sum_{\bgamma} |e_{i|\nu(\bgamma,\bbeta)}-e_{i|\nu(\bgamma^*,\bbeta^*)}|\pi(\bgamma|\bbeta,D_n)\pi(\bbeta|D_n) d\bbeta+\rho(4\epsilon_n)\\
   \leq& \hat\pi(4\epsilon_n)+\rho(4\epsilon_n) \stackrel{p}{\rightarrow} 0,
  \end{split}
 \end{equation}
 where $\stackrel{p}{\to}$ denotes convergence in probability, and 
$\hat\pi(c)$ denotes the posterior probability of the set $A(c)=\{\bbeta: d(p_\bbeta,p_{\mu^*})\geq c\}$. The last convergence is due to the identifiability condition B.1 and the posterior consistency result.  This completes the proof of part (i).

Part (ii) \& (iii):  They are directly implied by part (i).

\end{proof}

\subsection{Proof of Theorem 2.3}

\begin{proof}
Let $\bu=\sqrt{n}(\bbeta-\hat{\bbeta})=(u_1,\ldots,u_{K_n})^T$, and let $g(\bbeta)=nh(\bbeta)-nh(\hat{\bbeta})-\frac{n}{2}\sum h_{i,j}(\hat{\bbeta})(\beta_{i}-\hat{\beta}_{i})(\beta_{j}-\hat{\beta}_{j})$.
It is easy to see that for all $1\leq i_{1},\dots,i_{d}\leq K_{n}$, $g_{i_{1},\dots,i_{d}}(\hat{\bbeta})=0$ 
if $1\leq d\leq2$, and $g_{i_{1},\dots,i_{d}}(\bbeta)=nh_{i_{1},\dots,i_{d}}(\bbeta)$
if $d\geq3$.
 
Consider Taylor's expansions of $b(\bbeta)$  and  $\exp(g(\bbeta))$ at $\hat{\bbeta}$, we have
\[
\begin{split}
b(\bbeta)&=b(\hat{\bbeta})+\sum b_{i}(\hat{\bbeta})(\beta_{i}-\hat{\beta}_{i})+\frac{1}{2}\sum b_{i,j}(\tilde{\bbeta})(\beta_{i}-\hat{\beta}_{i})(\beta_{j}-\hat{\beta}_{j}) \\
=&b(\hat{\bbeta})+\frac{1}{\sqrt{n}}\sum b_{i}(\hat{\bbeta})u_{i}+\frac{1}{2n}\sum b_{i,j}(\tilde{\bbeta})u_{i}u_{j}, 
\end{split}
\]
\[
\begin{split}
e^{g(\bbeta)}&=1+\frac{n}{3!}\sum_{i,j,k} h_{i,j,k}(\hat{\bbeta})(\beta_{i}-\hat{\beta}_{i})(\beta_{j}-\hat{\beta}_{j})(\beta_{k}-\hat{\beta}_{k}) \\
&+\frac{n}{4!}e^{g({\check{\bbeta}})}\sum_{i,j,k,l} h_{i,j,k,l}({\check{\bbeta}})(\beta_{i}-\hat{\beta}_{i})(\beta_{j}-\hat{\beta}_{j})(\beta_{k}-\hat{\beta}_{k})(\beta_{l}-\hat{\beta}_{l}) \\
&=1+\frac{1}{6 \sqrt{n}}\sum h_{i,j,k}(\hat{\bbeta})u_{i}u_{j}u_{k}+\frac{1}{24n}e^{g({\check{\bbeta}})}\sum h_{i,j,k,l}(\check{\bbeta})u_{i}u_{j}u_{k}u_{l}, 
\end{split}
\]
where $\tilde{\bbeta}$ and $\check{\bbeta}$ are two points between $\bbeta$ and $\hat{\bbeta}$.
In what follows, we also treat $\tilde{\bbeta}$ and $\check{\bbeta}$ as functions of $\bu$, while treating $\hat{\bbeta}$ as a constant. 
Let $\phi(\bu) = \det(-\frac{1}{2\pi}H_n(\hat{\bbeta}))e^{\frac{1}{2}\sum h_{i,j}(\hat{\bbeta})u_{i}u_{j}}$ be the centered normal density with covariance matrix $-H_n^{-1}$. Then
\[
\begin{split}
 & \int_{B_{\delta(\bbeta)}}b(\bbeta)e^{nh(\bbeta)}d\bbeta 
 =e^{nh(\hat{\bbeta})}\int_{B_{\delta(\bbeta)}}e^{\frac{n}{2}\sum h_{i,j}(\hat{\bbeta})(\beta_{i}-\hat{\beta}_{i})(\beta_{j}-\hat{\beta}_{j})}b(\bbeta)e^{g(\bbeta)}d\bbeta \\
&=e^{nh(\hat{\bbeta})}\det(-\frac{n}{2\pi}H_{n}(\hat{\bbeta}))^{-\frac{1}{2}}\int_{B_{\sqrt{n}\delta}(0)}\phi(\bu)\left(b(\hat{\bbeta})+\sum\frac{1}{\sqrt{n}}b_{i}(\hat{\bbeta})u_{i}+\frac{1}{2}\frac{1}{n}\sum b_{i,j}(\tilde{\bbeta}(\bu))u_{i}u_{j}\right) \\
& \times \left(1+\frac{1}{6}\frac{1}{\sqrt{n}}\sum h_{i,j,k}(\hat{\bbeta})u_{i}u_{j}u_{k}+\frac{1}{24}\frac{1}{n}e^{g(\tilde{\bbeta}(\bu))}\sum h_{i,j,k,l}(\check{\bbeta}(\bu))u_{i}u_{j}u_{k}u_{l}\right)d \bu \\
& = e^{nh(\hat{\bbeta})}\det(-\frac{n}{2\pi}H_{n}(\hat{\bbeta}))^{-\frac{1}{2}}\int_{B_{\sqrt{n}\delta}(0)}\phi(\bu)(I_1+I_2)d \bu,
\end{split}
\]
where
\[
\begin{split}
    I_1 &= b(\hat{\bbeta}) + \frac{1}{n^{\frac{1}{2}}}\left(\sum b_{i}(\hat{\bbeta})u_{i}+\frac{ b(\hat{\bbeta})}{6}\sum h_{i,j,k}(\hat{\bbeta})u_{i}u_{j}u_{k}\right),\\
    I_2 &= \frac{1}{n}\left(b_{i}(\hat{\bbeta})u_{i}\frac{ 1}{6}\sum h_{i,j,k}(\hat{\bbeta})u_{i}u_{j}u_{k}\right)\\
    &+\frac{1}{2}\frac{1}{n}\sum b_{i,j}(\tilde{\bbeta}(\bu))u_{i}u_{j}\left(1+\frac{1}{6}\frac{1}{\sqrt{n}}\sum h_{i,j,k}(\hat{\bbeta})u_{i}u_{j}u_{k}+\frac{1}{24}\frac{1}{n}e^{g(\tilde{\bbeta}(u))}\sum h_{i,j,k,l}(\check{\bbeta}(\bu))u_{i}u_{j}u_{k}u_{l}\right)\\
    &+\frac{1}{24}\frac{1}{n}e^{g(\tilde{\bbeta}(\bu))}\sum h_{i,j,k,l}(\check{\bbeta}(\bu))u_{i}u_{j}u_{k}u_{l}\left(b(\hat{\bbeta})+\sum\frac{1}{\sqrt{n}}b_{i}(\hat{\bbeta})u_{i}\right),
\end{split}
\]
and we will study the two terms $\int_{ B_{\sqrt{n}\delta}(0)}\phi(\bu)I_1d\bu$ and $\int_{ B_{\sqrt{n}\delta}(0)}\phi(\bu)I_2d\bu$ separately.

To quantify the term $\int_{ B_{\sqrt{n}\delta}(0)}\phi(\bu)I_1d\bu$, we first bound $\int_{\mR^{K_{n}}\setminus B_{\sqrt{n}\delta}(0)}\phi(\bu)I_1d\bu$.
By assumption C.2 and the Markov inequality,
\begin{equation}\label{eql1}
\begin{split}
\int_{\mR^{K_{n}}\setminus B_{\sqrt{n}\delta}(0)}\phi(\bu)d\bu 
= P(\sum_{i=1}^{K_n}U_i^2 > n\delta^2)
\leq \frac{\sum_{i=1}^{K_n}E(U_i^2)}{n\delta^2} 
\leq \frac{r_nM+\frac{C(K_n-r_n)}{K_n^2}}{n\delta^2}
= O(\frac{r_n}{n}),
\end{split}
\end{equation}
where $(U_1,\dots,U_{K_n})^T$ denotes a multivariate normal random vector following  density $\phi(\bu)$.
Now we consider the term $\int_{\mR^{K_{n}}\setminus B_{\sqrt{n}\delta}(0)}\frac{1}{\sqrt{n}}\frac{ b(\hat{\bbeta})}{6}h_{i,j,k}(\hat{\bbeta})u_{i}u_{j}u_{k}\phi(\bu) d\bu$, 
by Cauchy-Schwarz inequality and assumption C.1, we have
\[
\begin{split}
    \left|\int_{\mR^{K_{n}}\setminus B_{\sqrt{n}\delta}(0)}\frac{1}{\sqrt{n}}\frac{ b(\hat{\bbeta})}{6}h_{i,j,k}(\hat{\bbeta})u_{i}u_{j}u_{k} \phi(\bu)d\bu \right|
    &= \left|E \left(\frac{1}{\sqrt{n}}\frac{ b(\hat{\bbeta})}{6} h_{i,j,k}(\hat{\bbeta})U_{i}U_{j}U_{k} \mathbbm{1}\left(\sum_{t=1}^{K_n}U_t^2 > n\delta^2\right)\right)\right|\\
    &\leq\sqrt{\frac{M_1}{n}E(U_i^2U_j^2U_k^2)P(\sum_{t=1}^{K_n}U_t^2 > n\delta^2)}\\
    &=O(\frac{\sqrt{r_n}}{n})\sqrt{E(U_i^2U_j^2U_k^2)}
\end{split}
\]
where $M_1$ is some constant. To bound $E(U_i^2U_j^2U_k^2)$,we refer to Theorem 1 of \cite{li2012gaussian}, which proved that for $1\leq i_{1},\dots,i_{6}\leq K_{n}$,
\[
\begin{split}
E(|U_{i_1}\dots U_{i_6}|)\leq\sqrt{\sum_{\pi}\prod_{j=1}^{6}h^{i_j,i_{\pi(j)}}},
\end{split}
\]
where the sum is taken over all permutations $\pi=(\pi(1),\dots , \pi(6))$ of set $\{1,\dots,6\}$ and $h^{i,j}$ is the $(i,j)$-th element of the covariance matrix $H^{-1}$. Let $m:=m(i_1,\dots,i_d)=|\{j:i_j\in \bgamma^*,j\in\{1,2,\dots,d\}\}|$ count the number of indexes belonging to the true connection set. Then, by condition C.2, we have
\[
\begin{split}
E(|U_{i_1}\dots U_{i_6}|)\leq\sqrt{C_0M^m(\frac{1}{K_n^2})^{6-m}}=O(\frac{1}{K_n^{6-m}}).
\end{split}
\]
The above inequality implies that $E(U_i^2U_j^2U_k^2) = O(\frac{1}{K_n^{6-2m_0}})$, where $m_0=|\{i,j,k\}\cap \bgamma^*|$. Thus, we have 
\begin{equation}\label{eql2}
\begin{split}
    &\left|\int_{\mR^{K_{n}}\setminus B_{\sqrt{n}\delta}(0)}\frac{1}{n^{\frac{1}{2}}}\sum\frac{ b(\hat{\beta})}{6}h_{i,j,k}(\hat{\bbeta})u_{i}u_{j}u_{k}\phi(\bu) d\bu \right|\\
    \leq& \sum_{m_0=0}^3 \left(\begin{array}{c}3\\m_0\end{array}\right)r_n^{m_0}(K_n-r_n)^{3-m_0}O(\frac{1}{K_n^{3-m_0}}\frac{\sqrt{r_n}}{n})
    = O(\frac{r_n^{3.5}}{n}).
\end{split}
\end{equation}
By similar arguments, we can get the upper bound of the term $\left|\int_{\mR^{K_{n}}\setminus B_{\sqrt{n}\delta}(0)}\sum_i(b_i(\hat\bbeta)u_i)/\sqrt{n}\phi(\bu)d\bu\right|$. Thus, we obtain that $\left|\int_{\mR^{K_{n}}\setminus B_{\sqrt{n}\delta}(0)}\phi(\bu)I_1d\bu\right| \leq O(\frac{r_n^{3.5}}{n})$. Due to the fact that $\int_{\mR^{K_{n}}}\phi(\bu)I_1d\bu = b(\hat{\bbeta})$, we have  $\int_{B_{\sqrt{n}\delta}(0)}\phi(\bu)I_1d\bu = b(\hat{\bbeta})+O(\frac{r_n^{3.5}}{n})$.

Due to assumption C.1 and the fact that $b_{i,j}\leq M$, within $B_{\sqrt{n}\delta}(0)$, each term in $I_2$ is trivially bounded by a polynomial of $|\bu|$, such as,
\[
\left|\frac{1}{2}\frac{1}{n}\sum b_{i,j}(\tilde{\bbeta}(\bu))u_{i}u_{j}\frac{1}{24}\frac{1}{n}e^{g(\check{\bbeta}(\bu))}\sum h_{i,j,k,l}(\check{\bbeta}(\bu))u_{i}u_{j}u_{k}u_{l}\right| \leq \frac{1}{48}M^{2}e^{M}\frac{1}{n^{2}}\sum |u_{i}u_{j}|\sum |u_{i}u_{j}u_{k}u_{l}|.
\]
Therefore, there exists a constant $M_0$ such that within $B_{\sqrt{n}\delta}(0)$,
\[
|I_2| \leq M_0\left(\frac{1}{n}\sum|u_i u_j| +\frac{1}{n}\sum|u_i u_j u_k u_l|+\frac{1}{n^{\frac{3}{2}}}\sum|u_i u_j u_k u_l u_s|+\frac{1}{n^2}\sum|u_i u_j u_k u_l u_s u_t|\right) := I_3,
\]
Then we have 
\[
\left|\int_{B_{\sqrt{n}\delta}(0)}\phi(\bu)I_2d\bu\right| \leq
\int_{B_{\sqrt{n}\delta}(0)}\phi(\bu)I_3d \bu \leq \int_{\mR^{K_{n}}}\phi(\bu)I_3d\bu,
\]

By the same arguments as used to bound $E(U_i^2U_j^2U_k^2)$, we can show that $$\int_{\mR^{K_{n}}}\phi(\bu)\frac{1}{n^2}\sum|u_i u_j u_k u_l u_s u_t|d\bu = O(\frac{r_n^6}{n^2})$$ holds. The rest terms in $\int_{\mR^{K_{n}}}\phi(\bu)I_3d\bu$ can be bounded by the same manner, and in the end we have  $\left|\int_{B_{\sqrt{n}\delta}(0)}\phi(\bu)I_2d\bu\right| \leq \int_{\mR^{K_{n}}}\phi(\bu)I_3d\bu = O(\frac{r_n^4}{n})$. 


Then
$\int_{B_{\delta(\bbeta)}}b(\bbeta)e^{nh(\bbeta)}d\bbeta=e^{nh(\hat{\bbeta})}\det(-\frac{n}{2\pi}H_{n}(\hat{\bbeta}))^{-\frac{1}{2}}(b(\hat{\bbeta})+O(\frac{r_n^{4}}{n}))$ holds.
Combining it with condition C.3 and the boundedness of $b$,  we get $\int b(\bbeta)e^{nh_{n}(\bbeta)}d\bbeta=e^{nh(\hat{\bbeta})}\det(-\frac{n}{2\pi}H_{n}(\hat{\bbeta}))^{-\frac{1}{2}}(b(\hat{\bbeta})+O(\frac{r_n^{4}}{n}))$.
With similar calculations, we can get
$\int e^{nh_{n}(\bbeta)}d\bbeta=e^{nh(\hat{\bbeta})}\det(-\frac{n}{2\pi}H_{n}(\hat{\bbeta}))^{-\frac{1}{2
}}(1+O(\frac{r_{n}^{4}}{n}))$.
Therefore,
\[
\frac{\int b(\bbeta)e^{nh_{n}(\bbeta)}d\bbeta}{\int e^{nh_{n}(\bbeta)}d\bbeta}=\frac{e^{nh(\hat{\bbeta})}\det(-\frac{n}{2\pi}H_{n}(\hat{\bbeta}))^{-\frac{1}{2}}(b(\hat{\bbeta})+O(\frac{r_{n}^{4}}{n}))}{e^{nh(\hat{\bbeta})}
\det(-\frac{n}{2\pi}H_{n}(\hat{\bbeta}))^{-\frac{1}{2
}}(1+O(\frac{r_{n}^{4}}{n}))}=b(\hat{\bbeta})+O(\frac{r_n^{4}}{n}).
\]
\end{proof}

\subsection{Proof of Lemma 2.1}

\begin{proof}
Consider the following prior setting: (i) $\lambda_n = K_n^{-(1+\tau')H_n}$,
(ii) $\log \frac{1}{\sigma_{0,n}} = H_n \log(K_n)$,
(iii) $\sigma_{1,n}=1$, and (iv) $\epsilon_n \geq \sqrt{\frac{(\frac{16}{\delta}+16)r_nH_n\log{K_n}}{n}}$.
Note that this setting satisfies all conditions of previous theorems. 
Recall that the marginal posterior inclusion probability is given by 
\[
q_i =\int \sum_{\bgamma} e_{i|\nu(\bgamma,\bbeta)} \pi(\bgamma|\bbeta,D_n)\pi(\bbeta|D_n) d\bbeta
:=\int \pi(\nu(\gamma_i)=1|\beta_i)\pi(\bbeta|D_n)d \bbeta.
\]
For any false connection $c_i\notin \bgamma_{*}$, we have $e_{i|\nu(\bgamma^*,\bbeta^*)}=0$ and 
\[
\nonumber
|q_i|=\int \sum_{\bgamma} |e_{i|\nu(\bgamma,\bbeta)}-e_{i|\nu(\bgamma^*,\bbeta^*)}|\pi(\bgamma|\bbeta,D_n)\pi(\bbeta|D_n)
d\bbeta \leq \hat\pi(4\epsilon_n)+\rho(4\epsilon_n).
\]
A straightforward calculation shows 
\begin{eqnarray}
\nonumber
\pi(\nu(\gamma_i)=1|\beta_i)&=&\frac{1}{1+\frac{1-\lambda_n}{\lambda_n}\frac{\sigma_{1,n}}{\sigma_{0,n}}\exp\left\{-\frac{1}{2} (\frac{1}{\sigma_{0,n}^{2}} - \frac{1}{\sigma_{1,n}^2})\beta_i^2\right\}}\\
\nonumber
&=&\frac{1}{1+\exp\left\{-\frac{1}{2}(K_n^{2H_n}-1)(\beta_i^2-\frac{(4+2\gamma^{'})H_n\log(K_n)+2\log(1-\lambda_n)}{K_n^{2H_n}-1})\right\} }.
\end{eqnarray}
Let $M_n=\frac{(4+2\tau')H_n\log(K_n)+2\log(1-\lambda_n)}{K_n^{2H_n}-1}$. Then, by Markov inequality, 
\begin{eqnarray}
\nonumber
P(\beta_i^2>M_n|D_n)=P\left(\pi(\nu(\gamma_i)=1|\beta_i) >1/2|D_n \right)\leq 2|q_i| \leq 2(\hat\pi(4\epsilon_n)+\rho(4\epsilon_n)).
\end{eqnarray}
Therefore,
\begin{eqnarray}
\nonumber
E(\beta_i^2|D_n) &\leq& M_n + \int_{\beta_i^2>M_n}\beta_i^2 \pi(\beta|D_n)d\beta \\
\nonumber
&\leq& M_n + \int_{M_n<\beta_i^2 < {M_n^{-\frac{2}{\delta}}}}\beta_i^2 \pi(\beta|D_n)d\beta + \int_{\beta_i^2 > {M_n^{-\frac{2}{\delta}}}} M_n{|\beta_i|}^{2+\delta}\pi(\beta|D_n)d\beta \\ 
\nonumber
&\leq& M_n + {M_n^{-\frac{2}{\delta}}}P(\beta_i^2 > M_n) + CM_n.
\end{eqnarray}
Since $\frac{1}{K_n^{2H_n}}\prec M_n \prec \frac{1}{K_n^{2H_n-1}}$,  $\epsilon_n \geq \sqrt{\frac{(\frac{16}{\delta}+16)r_nH_n\log{K_n}}{n}}$,  we have ${M_n^{-\frac{2}{\delta}}} e^{-n\epsilon_n^2/4} \prec \frac{1}{K_n^{2H_n-1}}$.
Thus
\begin{eqnarray}
\nonumber
P^{*}\left\{E(\beta_i^2|D_n) \prec  \frac{1}{K_n^{2H_n-1}}\right\} \geq P^{*}\left(\hat{\pi}(4\epsilon_n) < 2e^{-n\epsilon_n^2/4}\right) \geq 1-2e^{-n\epsilon_n^2/4}.
\end{eqnarray}

\end{proof}

\subsection{Verification of the Bounded Gradient Condition in Theorem 2.3} 

This section shows that with an appropriate choice of prior hyperparameters, the first and second order derivatives of $\pi(\nu(\gamma_i)=1|\beta_i)$ (i.e. the function $b(\bbeta)$ in Theorem 2.3) and the third and fourth order derivatives of $\log \pi(\bbeta)$ 
are all bounded with a high probability. Therefore, 
 the assumption C.1 in Theorem 2.3 is reasonable.

Under the same setting of the prior as that used in the proof of Lemma 2.1, we can  show that the derivative of $\pi(\nu(\gamma_i)=1|\beta_i)$ is bounded with a high probability. For notational simplicity, we suppress the subscript $i$ in what follows and let
\[
f(\beta)=\pi(\nu(\gamma)=1|\beta) = \frac{1}{1+\frac{1-\lambda_n}{\lambda_n}\frac{\sigma_{1,n}}{\sigma_{0,n}}\exp\left\{-\frac{1}{2} (\frac{1}{\sigma_{0,n}^{2}} - \frac{1}{\sigma_{1,n}^2})\beta^2\right\}} := \frac{1}{1+C_2\exp\{-C_1\beta^2\}},
\]
where $C_1 = \frac{1}{2} (\frac{1}{\sigma_{0,n}^{2}} - \frac{1}{\sigma_{1,n}^2})=\frac{1}{2}(K_n^{2H_n}-1)$ and $C_2 = \frac{1-\lambda_n}{\lambda_n}\frac{\sigma_{1,n}}{\sigma_{0,n}}=(1-\lambda_n)K_n^{(2+\tau')H_n}$. Then we have $C_1\beta^2=\log(C_2)+\log(f(\beta))-\log(1-f(\beta))$. With some algebra, we can show that
\[
\left|\frac{d f(\beta)}{d\beta}\right|=2\sqrt{C_1}f(\beta)(1-f(\beta))\sqrt{\log(C_2)+\log(f(\beta))-\log(1-f(\beta))},
\]
\[
\begin{split}
\frac{d^2 f(\beta)}{d\beta^2} = f(\beta)(1-f(\beta))( 2C_1+4C_1(\log(C_2)+\log(f(\beta))-\log(1-f(\beta)) )(1-2f(\beta))).
\end{split}
\]
By Markov inequality and Theorem 2.2, for the false connections, 
\[
P\left\{f(\beta)(1-f(\beta)) > \frac{1}{C_1^2}|D_n\right\} \leq P(f(\beta) > \frac{1}{C_1^2}|D_n) \leq C_1^2E(f(\beta)|D_n)\leq C_1^2(\hat\pi(4\epsilon_n)+\rho(4\epsilon_n))
\]
holds, and for the true connections,
\[
P\left\{f(\beta)(1-f(\beta)) > \frac{1}{C_1^2} |D_n\right\} \leq P(1-f(\beta) > \frac{1}{C_1^2}|D_n) \leq C_1^2E(1-f(\beta)|D_n)\leq C_1^2(\hat\pi(4\epsilon_n)+\rho(4\epsilon_n))
\]
holds. Under the setting of Lemma 2.1, by Theorem 2.1, it is easy to see that $C_1^2(\hat\pi(4\epsilon_n)+\rho(4\epsilon_n)) \rightarrow 0$ as $n\to \infty$, and thus $(f(\beta)(1-f(\beta)) < \frac{1}{C_1^2}$ with high probability. Note that $\frac{\log(C_2)}{C_1} \rightarrow 0$ and $|f(\beta)\log(f(\beta))| < \frac{1}{e}$. Thus, when $(f(\beta)(1-f(\beta)) < \frac{1}{C_1^2}$ holds, 
\[
\left|\frac{d f(\beta)}{d\beta}\right|\leq \sqrt{\frac{\log(C_2)}{C_1}+(f(\beta)(1-f(\beta))\log(f(\beta))-(f(\beta)(1-f(\beta))\log(1-f(\beta))},
\]
is bounded. Similarly we can show that $\frac{d^2 f(\beta)}{d\beta^2}$ is also bounded. In conclusion, $\frac{d f(\beta)}{d\beta}$ and $\frac{d^2 f(\beta)}{d\beta^2}$ is bounded with  probability $P\left\{f(\beta)(1-f(\beta)) \leq \frac{1}{C_1^2} |D_n\right\}$ which tends to 1 as $n\to \infty$.

Recall that $\pi(\beta)=\frac{1-\lambda_n}{\sqrt{2\pi}\sigma_{0,n}}\exp\{-\frac{\beta^2}{2\sigma_{0,n}^2}\} + \frac{\lambda_n}{\sqrt{2\pi}\sigma_{1,n}}\exp\{-\frac{\beta^2}{2\sigma_{1,n}^2}\}$. With some algebra, we can show
\[
\begin{split}
\frac{d^3\log(\pi(\beta))}{d\beta^3}
&=2(\frac{1}{\sigma_{0,n}^2}-\frac{1}{\sigma_{1,n}^2})\frac{df(\beta)}{d\beta} + (\frac{\beta}{\sigma_{0,n}^2}-\frac{\beta}{\sigma_{1,n}^2})\frac{d^2f(\beta)}{d\beta^2} \\
&= 4C_1\frac{df(\beta)}{d\beta} + 2\sqrt{C_1} \frac{d^2f(\beta)}{d\beta^2}\sqrt{\log(C_2)+\log(f(\beta))-\log(1-f(\beta))}.
\end{split}
\]
With similar arguments to that used for $\frac{d^2 f(\beta)}{d\beta^2}$ and $\frac{d^2 f(\beta)}{d\beta^2}$, 
we can make the term $f(\beta)(1-f(\beta))$ very small with a probability tending to 1. 
Therefore, $\frac{d^3\log(\pi(\beta))}{d\beta^3}$ is bounded with a probability tending to 1. Similarly we can bound $\frac{d^4\log(\pi(\beta))}{d\beta^4}$ with a high
probability. 

\subsection{Approximation of Bayesian Evidence} 
\label{sectB.5}

 In Algorithm 1, each sparse model is evaluated by its Bayesian evidence:
\[
{Evidence} = \det(-\frac{n}{2\pi}H_{n}({\bbeta}_{\bgamma}))^{-\frac{1}{2}}e^{nh_{n}(
{\bbeta}_{\bgamma})},
\]
where $H_n({\bbeta}_{\bgamma})=\frac{\partial^2 h_n({\bbeta}_{\bgamma})}{\partial{\bbeta}_{\bgamma}\partial^T{\bbeta}_{\bgamma}}$ is the Hessian matrix, ${\bbeta}_{\bgamma}$ denotes the vector of  connection weights selected by the model $\bgamma$, i.e. $H_n(\bbeta_{\bgamma})$ is a $|\bgamma|\times |\bgamma|$ matrix, and the prior $\pi(\bbeta_{\bgamma})$ in $h_n(\bbeta)$ is only for the connection weights selected by the model $\bgamma$. 
Therefore, 
\begin{equation} \label{BICapprox}
\log(Evidence) = nh_{n}({\bbeta}_{\bgamma}) - \frac{1}{2} |\bgamma|\log(n) + \frac{1}{2} |\bgamma| \log(2\pi) -\frac{1}{2} \log(\det (-H_{n}({\bbeta}_{\bgamma}))),
\end{equation}
and $-H_{n}({\bbeta}_{\bgamma}) = -\frac{1}{n} \sum_{i=1}^{n}\frac{\partial^2 \log(p(y_{i},\bx_{i}|\bbeta_{\bgamma})) }{\partial\bbeta_{\bgamma}\partial^T\bbeta_{\bgamma}}-\frac{1}{n}\sum_{i=1}^{n}\frac{\partial^2 \log(\pi(\bbeta_{\bgamma})) }{\partial\bbeta_{\bgamma}\partial^T\bbeta_{\bgamma}}$. 
For the selected connection weights, the prior $\pi(\bbeta_{\bgamma})$ behaves like $N(0,\sigma_{1,n}^2)$, and then $-\frac{\partial^2 \log(\pi(\bbeta_{\bgamma})) }{\partial\bbeta_{\bgamma}\partial^T\bbeta_{\bgamma}}$ is a diagonal matrix with the diagonal elements approximately equal to $\frac{1}{\sigma_{1,n}^2}$ and $\frac{1}{n\sigma_{1,n}^2} \rightarrow 0$.

If $(y_i,\bx_i)$'s are viewed as i.i.d samples drawn from $p(y,\bx|\bbeta_{\bgamma}))$, then $-\frac{1}{n} \sum_{i=1}^{n}\frac{\partial^2 \log(p(y_{i},\bx_{i}|\bbeta_{\bgamma})) }{\partial\bbeta_{\bgamma}\partial^T\bbeta_{\bgamma}}$ will converge to the Fisher information matrix $ I(\bbeta_{\bgamma}) =  E(-\frac{\partial^2 \log(p(y,\bx|\bbeta_{\bgamma})) }{\partial\bbeta_{\bgamma}\partial^T\bbeta_{\bgamma}})$. If we further assume that $I(\bbeta_{\bgamma})$ has bounded eigenvalues, i.e. $C_{min}\leq \lambda_{min}(I(\bbeta_{\bgamma})) \leq  \lambda_{max}(I(\bbeta_{\bgamma}))\leq C_{max}$ for  some constants $C_{min}$ and $C_{max}$, then $\log(\det (-H_{n}({\bbeta}_{\bgamma}))) \asymp |\bgamma|$.
 This further implies 
 \[
 \frac{1}{2} |\bgamma| \log(2\pi) -\frac{1}{2} \log(\det (-H_{n}({\bbeta}_{\bgamma}))) \prec |\bgamma| \log(n).
 \]
 By keeping only the dominating terms in (\ref{BICapprox}), we have 
\[
\log(Evidence) \approx nh_n(\bbeta_{\bgamma}) - \frac{1}{2} |\bgamma|\log(n) = -\frac{1}{2}BIC + \log(\pi(\bbeta_{\bgamma})).
\]
Since we are comparing a few low-dimensional models (with the model size $|\bgamma|\prec n$), it is intuitive to  further ignore the prior term $\log(\pi(\bbeta_{\bgamma}))$. As a result, we can elicit the low-dimensional sparse neural networks by BIC.

\section{Proofs on Generalization Bounds}

\subsection{Proof of Theorem 2.4} 

\begin{proof}
Consider the set $B_n$ defined in (\ref{bset}). By the argument used in the proof of Theorem \ref{thm}, there exists a class of $\tilde B_n=\{\bbeta^{(l)}: 1\leq l<L\}$ for some $L<\exp\{cn\epsilon_n^2\}$ with a constant $c$ such that for any $\bbeta\in B_n$, there exists some $\bbeta^{(l)}$ satisfying $|\mu(\bbeta,\bx)-\mu(\bbeta^{(l)},\bx)|\leq c'\epsilon_n$. 

Let $\tilde\pi$ be the truncated distribution of $\pi(\bbeta|D_n)$ on $B_n$, and let $\check\pi$ be a discrete distribution on $\tilde B_n$ defined as 
$\check\pi(\bbeta^{(l)}) = \tilde\pi(B_l)$, where $B_l=\{\bbeta\in B_n: \|\mu(\bbeta,\bx)-\mu(\bbeta^{(l)},\bx)\|_{\infty}<\min_{j\neq l} \|\mu(\bbeta,\bx)-\mu(\bbeta^{(j)},\bx)\|_{\infty}\}$ by defining the norm $\|f\|_\infty=\max_{\{\bx\in\Omega\}}f(\bx)$.
Note that for any $\bbeta\in B_l$,  $\|\mu(\bbeta,\bx)-\mu(\bbeta^{(l)},\bx)\|\leq c'\epsilon_n$, thus,
\begin{equation}\label{ge0}
    l_0(\bbeta,\bx,y)\leq l_{c'\epsilon_n/2}(\bbeta^{(l)},\bx,y)\leq l_{c'\epsilon_n}(\bbeta,\bx,y).
\end{equation}

The above inequality implies that 
\begin{equation}\label{ge1}
\begin{split}  
   & \int E_{\bx,y}l_0(\bbeta,\bx,y)d\tilde\pi\leq  \int E_{\bx,y}l_{c'\epsilon_n/2}(\bbeta^{(l)},\bx,y)d\check\pi,\\
     & \int\frac{1}{n}\sum_{i=1}^nl_{c'\epsilon_n/2}(\bbeta^{(l)},\bx^{(i)},y^{(i)})d\check\pi\leq  \int \frac{1}{n}\sum_{i=1}^nl_{c'\epsilon_n}(\bbeta,\bx^{(i)},y^{(i)})d\tilde\pi.
\end{split}
\end{equation}

Let $P$ be a uniform prior on $\tilde B_n$, by Lemma 2.2, with probability $1-\delta$,
\begin{equation}\label{ge2}
\begin{split}
\int E_{\bx,y}l_\nu(\bbeta^{(l)},\bx,y)d\check\pi\leq& \int \frac{1}{n}\sum_{i=1}^nl_\nu(\bbeta^{(l)},\bx^{(i)},y^{(i)})d\check\pi+\sqrt{\frac{d_0(\check\pi,P)+\log\frac{2\sqrt n}{\delta}}{2n}}\\
\leq& \int \frac{1}{n}\sum_{i=1}^nl_\nu(\bbeta^{(l)},\bx^{(i)},y^{(i)})d\check\pi+\sqrt{\frac{cn\epsilon_n^2+\log\frac{2\sqrt n}{\delta}}{2n}},
\end{split}
\end{equation}
for any $\nu\geq0$ and $\delta>0$, \textcolor{black}{where the second inequality is due to the fact that
$d_0(\mathcal L,P)\leq \log L$ for any discrete distribution $\mathcal L$ over $\{\bbeta^{(l)}\}_{l=1}^L$.}

Combining inequalities (\ref{ge1}) and (\ref{ge2}), we have that, with probability $1-\delta$,
\begin{equation}\label{ge3}
     \int E_{\bx,y}l_0(\bbeta,\bx,y)d\tilde\pi\leq  \sqrt{\frac{cn\epsilon_n^2+\log\frac{2\sqrt n}{\delta}}{2n}}+\int \frac{1}{n}\sum_{i=1}^nl_{c'\epsilon_n}(\bbeta,\bx^{(i)},y^{(i)})d\tilde\pi.
\end{equation}
Due to the boundedness of $l_\nu$, we have 
\begin{equation}
    \begin{split}
         & \int E_{\bx,y}l_0(\bbeta,\bx,y)d\pi(\bbeta|D_n)\leq \int E_{\bx,y}l_0(\bbeta,\bx,y)d\tilde\pi+\pi(B_n^c|D_n),\\
        & \int \frac{1}{n}\sum_{i=1}^nl_{c'\epsilon_n}(\bbeta,\bx^{(i)},y^{(i)})d\tilde\pi\leq \frac{1}{1-\pi(B_n^c|D_n)}\int \frac{1}{n}\sum_{i=1}^nl_{c'\epsilon_n}(\bbeta,\bx^{(i)},y^{(i)})d\pi(\bbeta|D_n).
    \end{split}
\end{equation}
Note that the result of Theorem A.1 implies that, with probability at least $1-\exp\{-c''n\epsilon_n^2\}$, $\pi(B_n^c|D_n)\leq 2\exp\{-c''n\epsilon_n^2\}$.
Therefore, with probability greater than $1-\delta-\exp\{-c''n\epsilon_n^2\}$,
\[\begin{split}
    \int E_{\bx,y}l_0(\bbeta,\bx,y)d\pi(\bbeta|D_n)\leq&\frac{1}{1-2\exp\{-c''n\epsilon_n^2\}}\int \frac{1}{n}\sum_{i=1}^nl_{c'\epsilon_n}(\bbeta,\bx^{(i)},y^{(i)})d\pi(\bbeta|D_n)\\
    &+\sqrt{\frac{cn\epsilon_n^2+\log\frac{2\sqrt n}{\delta}}{2n}}+2\exp\{-c''n\epsilon_n^2\}.
\end{split}
\]
Thus, the result holds if we choose $\delta=\exp\{-c'''n\epsilon_n^2\}$ for some $c'''$.
\end{proof}

\subsection{Proof of Theorem 2.5}

\begin{proof} To prove the theorem, we first introduce a lemma on generalization error of finite classifiers, which can be easily derived based on Hoeffding's inequality:

\begin{lemma}[Generalization error for finite classifier]\label{pac-2}
Given a set $B$ which contains $H$ elements, if the estimator $\hat\bbeta$ belongs to $B$ and the loss function $l\in[0,1]$, then with probability $1-\delta$,
\[
E_{\bx,y}l(\hat\bbeta,\bx,y)\leq  \frac{1}{n}\sum_{i=1}^nl(\hat\bbeta,\bx^{(i)},y^{(i)})+\sqrt{\frac{\log H+\log(1/\delta)}{2n}}.
\]
\end{lemma}

Next, let's consider the same sets $B_n$ and $\tilde B_n$ as defined in the proof of Theorem 2.4. Due to the posterior contraction result, 
with probability at least $1-\exp\{-c''n\epsilon_n^2\}$, the estimator $\hat\bbeta\in B_n$. Therefore, there must exist some $\bbeta^{(l)}\in\tilde B_n$ such that 
(\ref{ge0}) holds, which implies that with probability at least $1-\exp\{-c''n\epsilon_n^2\}-\delta$,
\[\begin{split}
L_0(\hat\bbeta)&\leq L_{c'\epsilon_n/2}(\bbeta^{(l)})\leq L_{emp,c'\epsilon_n/2}(\bbeta^{(l)})+\sqrt{\frac{\log H+\log(1/\delta)}{2n}}\\
&\leq L_{emp,c'\epsilon_n}(\hat\bbeta)+\sqrt{\frac{\log H+\log(1/\delta)}{2n}},
\end{split}\]
where the second inequality is due to Lemma \ref{pac-2}, and $H\leq \exp\{cn\epsilon_n^2\}$. 
The result then holds if we set $\delta=\exp\{-cn\epsilon_n^2\}$.
\end{proof}

\subsection{Proof of Theorems 2.6 and 2.7}

The proofs are straightforward and thus omitted. 



\section{Mathematical facts of sparse DNN}
Consider a sparse DNN model with $H_n-1$ hidden layer. Let $L_1,\dots, L_{H_n-1}$ denote the number of node in each hidden layer and $r_i$ be the number of active connections that connect {\it to} the $i$th hidden layer (including the bias for the $i$th hidden layer and weight connections between $i-1$th and $i$th layer). Besides, we let $O_{i,j}(\bbeta, \bx)$ denote the output value of the $j$th node in the $i$th hidden layer

\begin{lemma}\label{dnn1}
Under assumption A.1, if a sparse DNN has at most $r_n$ connectivity (i.e., $\sum r_i=r_n$), and all the weight and bias parameters are bounded by $E_n$ (i.e., $\|\bbeta\|_\infty\leq E_n$), then
the summation of the outputs of the $i$th hidden layer for $1\leq i\leq H_n$ is bounded by 
\[\sum_{j=1}^{L_i}O_{i,j}(\bbeta,\bx)\leq E_n^i\prod_{k=1}^i r_k,\]
where the $H_n$-th hidden layer means the output layer.
\end{lemma}
\begin{proof}
For the simplicity of representation, we rewrite $O_{i,j}(\bbeta,\bx)$ as $O_{i,j}$ when causing no confusion.
The lemma is the result from the facts that 
\[
\begin{split}
    &\sum_{i=1}^{L_1}|O_{i,j}|\leq r_nE_n,\mbox{ and }
    \sum_{j=1}^{L_i}|O_{i,j}|\leq \sum_{j=1}^{L_{i-1}}|O_{i-1,j}|E_n r_i.
\end{split}
\]
\end{proof}

Consider two neural networks, $\mu(\bbeta,\bx)$ and $\mu(\widetilde{\bbeta},\bx)$, where the formal one is a sparse network satisfying $\|\bbeta\|_0=r_n$ and $\|\bbeta\|_\infty=E_n$, and its model vector is $\bgamma$. If $|\bbeta_i-\tilde\bbeta_i|<\delta_1$ for all $i\in\bgamma$ and $|\bbeta_i-\tilde\bbeta_i|<\delta_2$ for all $i\notin\bgamma$, then
\begin{lemma}\label{dnn2}
\[
\max_{\|\bx\|_\infty\leq 1}|\mu(\bbeta,\bx)-\mu(\widetilde\bbeta,\bx)|\leq 
\delta_1H_n(E_n+\delta_1)^{H_n-1}\prod_{i=1}^{H_n} r_i+\delta_2(p_nL_1+\sum_{i=1}^{H_n}L_i)\prod_{i=1}^{H_n}[(E_n+\delta_1)r_i+\delta_2L_i].\]
\end{lemma}
\begin{proof}
Define $\check\bbeta$ such that 
 $\check\bbeta_i=\tilde\bbeta_i$ for all $i\in\bgamma$ and $\check\bbeta_i=0$ for all $i\notin\bgamma$.
Let $\check O_{i,j}$ denote $O_{i,j}(\check\bbeta,\bx)$.
Then,
\[
\begin{split}
|\check O_{i,j}-O_{i,j}|&\leq \delta_1\sum_{j=1}^{L_{i-1}}|O_{i-1,j}|+E_n\sum_{j=1}^{L_{i-1}}|\check O_{i-1,j}-O_{i-1,j}|+\delta_1\sum_{j=1}^{L_{i-1}}|\check O_{i-1,j}-O_{i-1,j}|\\
&\leq \delta_1\sum_{j=1}^{L_{i-1}}|O_{i-1,j}|+(E_n+\delta_1)\sum_{j=1}^{L_{i-1}}|\check O_{i-1,j}-O_{i-1,j}|.
\end{split}
\]
This implies a recursive result
\[
\sum_{j=1}^{L_i}|\check O_{i,j}-O_{i,j}|
\leq r_i(E_n+\delta_1)\sum_{j=1}^{L_{i-1}}|\check O_{i-1,j}-O_{i-1,j}|+r_i\delta_1\sum_{j=1}^{L_{i-1}}|O_{i-1,j}|.
\]
Due to Lemma \ref{dnn1}, $\sum_{j=1}^{L_{i-1}}|O_{i-1,j}|\leq E_n^{i-1}r_1\cdots r_{i-1}$. Combined with the fact that 
$\sum_{j=1}^{L_1}|\check O_{1,j}-O_{1,j}|\leq \delta_1r_1$,
one have that 
\[
|\mu(\bbeta,\bx)-\mu(\check\bbeta,\bx)|=\sum_j|\check O_{H_n,j}-O_{H_n,j}|\leq \delta_1H_n(E_n+\delta_1)^{H_n-1}\prod_{i=1}^{H_n} r_i.
\]
Now we compare $\widetilde O_{i,j}:=\mu(\widetilde\bbeta,\bx)$ and $\check O_{i,j}$.
We have that
\[
\sum_{j=1}^{L_i} |\widetilde O_{i,j}-\check O_{i,j}|\leq 
\delta_2L_i\sum_{j=1}^{L_{i-1}} |\widetilde O_{i-1,j}-\check O_{i-1,j}|+\delta_2L_i\sum_{j=1}^{L_{i-1}} |\check O_{i-1,j}|+r_i(E_n+\delta_1)\sum_{j=1}^{L_{i-1}} |\widetilde O_{i-1,j}-\check O_{i-1,j}|,
\]
and 
\[
\sum_{j=1}^{L_1} |\widetilde O_{1,j}-\check O_{1,j}|\leq \delta_2p_nL_1.
\]
Due to Lemma \ref{dnn1}, we also have that $\sum_{j=1}^{L_{i-1}}|\check O_{i-1,j}|\leq (E_n+\delta_1)^{i-1}r_1\cdots r_{i-1}$. 
Together, we have that 
\[\begin{split}
    &|\mu(\widetilde\bbeta,\bx)-\mu(\check\bbeta,\bx)|=\sum_j|\widetilde O_{H_n,j}-\check O_{H_n,j}|\\
    \leq&\delta_2(p_nL_1+\sum_{i=1}^{H_n}L_i)\prod_{i=1}^{H_n}[(E_n+\delta_1)r_i+\delta_2L_i].
\end{split}
\]
The proof is concluded by summation of the bound for $|\mu(\bbeta,\bx)-\mu(\check\bbeta,\bx)|$ and $|\mu(\widetilde\bbeta,\bx)-\mu(\check\bbeta,\bx)|$.
\end{proof}


\bibliographystyle{asa}
\bibliography{Reference}

\end{document}